\documentclass[british, twoside]{article}
\usepackage{amsmath}
\usepackage[ruled,vlined]{algorithm2e}
\SetKw{KwBy}{by}

\usepackage{amsthm}
\usepackage{thm-restate}
\usepackage{thmtools}

\usepackage[accepted]{aistats2021}

\newtheorem{prop}{Proposition}

\newtheorem{definition}{Definition}[section]

 \usepackage[utf8]{inputenc}
\usepackage{parskip}    %

\usepackage{graphicx}
\usepackage{amssymb}
\usepackage{epstopdf}
\usepackage[round]{natbib}

\usepackage{accents}

\usepackage[justification=centering]{subfig}
\usepackage{mathrsfs}
\usepackage[toc,page]{appendix}
\usepackage{pdfpages}
\usepackage{attachfile2}
\usepackage[abs]{overpic}
\usepackage{pict2e}
\usepackage{algpseudocode}
\usepackage{cancel}
\usepackage{tikz}

\usetikzlibrary{bayesnet}
\usetikzlibrary{arrows,decorations.markings}
\usepackage[english]{isodate}
\usepackage{hyperref}
\usepackage{booktabs} %
\usepackage{xcolor}
\usepackage[colorinlistoftodos,prependcaption,textsize=tiny]{todonotes}
\usepackage{mathtools}
\usepackage{mismath}
\usepackage{stackengine}

\usepackage{todonotes}
\usepackage{wrapfig}
\usepackage{mdframed}

\usetikzlibrary{patterns}
\usetikzlibrary{bayesnet}
\usetikzlibrary{arrows,decorations.markings}
\tikzstyle{disent_latent} = [circle,pattern=north east lines, pattern color=black!20,draw=black,inner sep=1pt,
minimum size=20pt, font=\fontsize{10}{10}\selectfont, node distance=1]

\DeclareMathOperator{\expect}{\mathbb{E}}

\DeclareMathOperator{\ELBO}{\mathcal{L}}
\DeclareMathOperator{\KL}{\mathrm{KL}}
\DeclareMathOperator{\intd}{\mathrm{d}}

\DeclarePairedDelimiterX\MeijerM[3]{\lparen}{\rparen}%
{\begin{smallmatrix}#1 \\ #2\end{smallmatrix}\delimsize\vert\,#3}

\DeclareMathOperator*{\argmax}{arg\,max}
\DeclareMathOperator*{\argmin}{arg\,min}

\newcommand{\vecto}[1]{\boldsymbol{\mathbf{#1}}}
\renewcommand{\v}{\vecto}

\newcommand{\intdx}{\mathbin{{\intd}{\v{x}}}}
\newcommand{\intdz}{\mathbin{{\intd}{\v{z}}}}

\fancypagestyle{appendixpage}{
    \fancyhead{}
    \renewcommand{\headrulewidth}{0pt}
    \renewcommand{\footrulewidth}{0pt}}

\begin{document}
\runningtitle{Towards a Theoretical Understanding of the Robustness of Variational Autoencoders}
\runningauthor{Alexander Camuto, Matthew Willetts, Stephen Roberts, Chris Holmes, Tom Rainforth}

\twocolumn[

\aistatstitle{Towards a Theoretical Understanding of the \\ Robustness of Variational Autoencoders}

\aistatsauthor{Alexander Camuto$^{1,2}$ \hspace{1.8cm} Matthew Willetts$^{1,2}$}
\aistatsauthor{Stephen Roberts$^{1,2}$ \hspace{1.2cm}
Chris Holmes$^{1,2}$ \hspace{1.2cm} Tom Rainforth$^{1}$}
\aistatsaddress{${}^{1}$University of Oxford \hspace{1cm} ${}^{2}$Alan Turing Institute}]

\begin{abstract}
We make inroads into understanding the robustness of Variational Autoencoders (VAEs) to adversarial attacks and other input perturbations.
While previous work has developed algorithmic approaches to attacking and defending VAEs, there remains a lack of formalization for what it means for a VAE to be robust.
To address this, we develop a novel criterion for robustness in probabilistic models: $r$-robustness.
We then use this to construct the first theoretical results for the robustness of VAEs, deriving margins in the input space for which we can provide guarantees about the resulting reconstruction.
Informally, we are able to define a region within which any perturbation will produce a reconstruction that is similar to the original reconstruction.
To support our analysis, we show that VAEs trained using disentangling methods not only score well under our robustness metrics, but that the reasons for this can be interpreted through our theoretical results.

\end{abstract}

\section{Introduction}
\label{intro}

Variational Autoencoders (VAEs)~\citep{Rezende2014, Kingma2013} have been found to be more robust to input perturbations than their deterministic counterparts, particularly those originating from adversarial attacks \citep{Willetts2019a, Schott2019, Ghosh2018}. 
This trait has made them useful in protecting downstream tasks~\citep{Willetts2019a, Schott2019, Ghosh2018}. 

Nevertheless, they are still not completely impervious to attack~\citep{Tabacof2016,Gondim-ribeiro, Kos2016}: a hypothetical adversary can attack a VAE by applying small input perturbations to invoke meaningful changes in the encoding.
Typically this is done by trying to find perturbations which produce reconstructions close to a distinct target datapoint chosen by the adversary, rather than being representative of the original input.
Such attacks have been shown to be successful in a wide range of scenarios~\citep{Tabacof2016,Gondim-ribeiro, Kos2016,Willetts2019a}.
Recent work has made progress towards defending against them from an empirical and algorithmic perspective~\citep{Willetts2019a}, by repurposing approaches designed to learn disentangled latent representations~\citep{Burgess2018, Chen2018, Mathieu2019}.

However, a deeper understanding of the mechanisms underpinning the robustness of VAEs and their derivatives
is still lacking.
Furthermore, there are currently no theoretical foundations for this robustness or even any frameworks or formalizations for exactly what it means for a VAE to be ``robust.''
In other words, what would it mean to have a certifiably--robust VAE?
Moreover, are there scenarios where we might be able to provide theoretical guarantees of such robustness?

\begin{figure*}[t]
    \centering
     \includegraphics[width=0.8\textwidth]{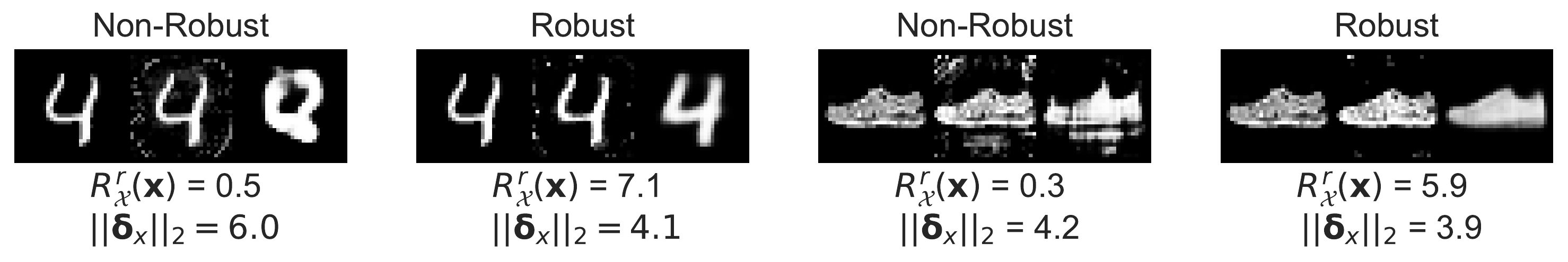}
\vspace{-8pt}
    \caption{Reconstructions under attack for robust and non-robust VAEs.
    Each subfigure shows from left to right: the original input, a perturbed input made by an adversarial attack, and the reconstruction of the perturbed input. 
    We show results for VAEs that are robust ($R^r_{\mathcal{X}}(\v{x})\ge\lVert\v \delta_x\rVert_2$) and non-robust ($R^r_{\mathcal{X}}(\v{x})<\lVert\v \delta_x\rVert_2$) for a given point $\v{x}$ and adversarially selected perturbation $\v \delta_x$.
    We see that the robust VAE reconstructions are visually closer to the original input.
    }
    \label{fig:robust_example}
    \vspace{-8pt}
\end{figure*}
As a first step to addressing these questions, we develop the first metric with which to evaluate the robustness of VAEs: $r$-robustness. 
Informally, for a given input, a VAE is $r$-robust to a given perturbation if it is more likely that its reconstruction will fall within a ball of radius $r$ around the undistorted maximum likelihood reconstruction, than outside it.
The smaller the value of $r$ for which we can confirm $r$-robustness, the more robust we can guarantee the VAE to be.
Through $r$-robustness, we provide theoretical foundations to understand the source of VAEs' robustness and provide insights into what can cause them to more or less robust.

Using this, we next develop a \emph{margin} of robustness, $R^r_{\mathcal{X}}(\v x)$, such that the VAE is $r$-robust to \emph{any} possible perturbations of the input $\v x$ within this margin.
This, in turn, allows us to provide a notion of a certifiably--robust reconstruction as it forms a guarantee that no attack limited to the margin can reliably undermine it.
An example of this is shown in Fig \ref{fig:robust_example}, where we demonstrate that large $R^r_\mathcal{X}(\v x)$ are associated with model--input pairs that are robust to adversarially generated input perturbations.
Analogously to the concept of an adversarial risk~\citep{uesato2018adversarial}, $R^r_{\mathcal{X}}(\v x)$ can further be converted to a metric for the \emph{overall} robustness of a VAE, by taking its expectation over the data generating distribution.

To make inroads towards imposing apriori constraints on a VAE that ensure that it is certifiably robust, we further derive a theoretical bound for $R^r_{\mathcal{X}}(\v x)$ as a function of the encoder variance and Jacobian.
This provides insights into the characteristics of VAEs that contribute to robustness.
Building on this result, we show empirically that VAEs with larger encoder variances and smaller Jacobians typically produce larger margins $R^r_{\mathcal{X}}(\v{x})$ and are thus more robust to perturbations.
We further demonstrate how these beneficial characteristics can be induced using methods introduced to learn disentangled representations, deriving new results for how these methods can be interpreted.

To summarize, our core contributions are that we first define a robustness metric, $r$-robustness, that is tailored to probabilistic generative models.
We develop a margin $R^r_{\mathcal{X}}(\v{x})$ on a VAE's input space within which it is $r$-robust to perturbations. 
Finally, we offer theoretical and empirical analysis---based on $R^r_{\mathcal{X}}(\v{x})$ and existing disentanglement methods---that can aid the construction of robust VAEs.

\section{Background}

\subsection{Variational Autoencoders}
VAEs~\citep{Rezende2014, Kingma2013}, and the models they have inspired~\citep{Alemi2016,Higgins2017,Chen2018,Willetts2019a}, are deep latent variable models.
Using $\v{x}\in\mathcal{X}$ to denote data and $\v{z}\in\mathcal{Z}$ to denote the latents with associated prior $p(\v{z})$, a VAE simultaneously learns both a forward generative model, $p_\theta(\v{x}|\v{z})$, and an amortised approximate posterior distribution, $q_\phi(\v{z}|\v{x})$, where $\theta$ and $\phi$ correspond to their respective parameters, typically taking the form of neural networks.
These are referred to as the decoder and encoder respectively, and a VAE can be thought of as a deep stochastic autoencoder.
Under this autoencoder framework, one typically takes the reconstructions as deterministic, corresponding to the mean of the decoder, namely $g_{\theta}(\v z) := \expect_{p_\theta(\v{x}|\v{z})}[\v x]$, a convention we adopt.

A VAE is trained by maximizing the 
evidence lower bound (ELBO)
$\mathcal{L}\!=\! \expect_{p_{\mathcal{D}}(\v{x})}[\mathcal{L}(\v{x})]$, where
$$\ELBO(\v{x})=\expect_{q_\phi(\v{z}|\v{x})} \left[\log p_\theta(\v{x}|\v{z})\right] - \KL(q_\phi(\v{z}|\v{x}) || p(\v{z}))$$
and $p_{\mathcal{D}(\v{x})}$ represents the empirical data distribution.
The optimization is carried out using stochastic gradient descent with Monte Carlo samples, typically employing the reparameterization trick~\citep{Kingma2013}. 
For example, for a Gaussian 
$q_\phi(\v{z}|\v{x})$, we draw samples as \(\v{z} = \v{\mu}_\phi(\v{x}) + \v{\eta} \circ \v{\sigma}_\phi(\v{x}), \v{\eta} \sim \mathcal{N}(\v{0}, \v{I})\), where $\circ$ is the element-wise product.

\subsection{Adversarial Attacks for VAEs}

In adversarial settings, an agent is trying to alter the behavior of a model towards a specific goal. 
This could involve, in the case of classification, adding a very small perturbation to an input so as to alter the model's predicted class. 
For many deep learning models, small changes to data imperceptible to the human eye, can drastically change a model's output.

Proposed attacks for VAEs aim to produce reconstructions close to a chosen target image by applying small distortions to the input of a VAE~\citep{Tabacof2016, Gondim2018, Kos2018}.
The adversary optimizes this perturbation to minimize some measure of distance between the reconstruction and the target image \textit{or} the distance between the embedding of the distorted image and the embedding of the target image.

\subsection{Disentangled VAEs}

Learning \emph{disentangled} representations~\citep{bengio2013representation} involves training a probabilistic generative model in a manner that encourages a one-to-one correspondence between dimensions of the learnt latent space and some interpretable aspect of the data~\citep{Higgins2017,Alemi2016,Burgess2018, Chen2018, Mathieu2019}.
For instance, there might be some axis in latent space encoding `hair color' for images of people.
One such method is the $\beta$-VAE \citep{Burgess2018, Higgins2017}, which upweights the $\KL$ in the ELBO with a penalization factor $\beta$:
\[
\ELBO_{\beta}(\v{x})=\expect_{q_\phi(\v{z}|\v{x})}[\log p_\theta(\v{x}|\v{z})] - \beta \KL(q_\phi(\v{z}|\v{x})||p(\v{z}))]
\]
Disentangling can be difficult to achieve in practice, requiring careful hyperparameter tuning
\citep{Locatello2019, Mathieu2019, Rolinek2019}. 

Nevertheless, models trained under disentangling objectives have other beneficial properties.
In particular, $\beta$-TC regularization \citep{Chen2018}, another method proposed for disentangling, has been shown to induce models that are more robust to adversarial attacks \citep{Willetts2019a}.
The encoders of $\beta$-VAEs have also been used as the perceptual part of Deep RL models to create more robust agents \citep{Higgins2017a}.
Thus, regardless of the presence of disentangled generative factors, these regularization methods can induce models that are more robust to attack.

\section{Robustness of VAEs}
\label{sec:stoch_layers}

\subsection{A Probabilistic Metric of Robustness}

Deep learning models can be brittle.
Some of the most sophisticated deep learning classifiers can be broken by simply adding small pertubations to their inputs \citep{Szegedy2014, Shamir2019, Goodfellow2015, Papernot2016, Moosavi-Dezfooli2016}.
Here, perturbations that would not fool a human break neural network predictions.
A model's weakness to such perturbations is called its \textit{sensitivity}.
For classifiers, we can straightforwardly define an associated \emph{sensitivity margin}: it is the radius of the largest metric ball centered on an input $\v{x}$ for which a classifier's original prediction holds for all possible perturbations within that ball.

Defining such a margin for VAEs is conceptually more difficult as, in general, the reconstructions are continuous rather than discrete.
To put it another way, there is no step-change in VAE reconstructions that is akin to a change of a predicted class in classifiers;
\textit{any} perturbation in the input space will result in a change in the VAE output.
To complicate matters further, a VAE's latent space is stochastic: the same input can result in different reconstructions.

As a first step to deriving robustness margins for VAEs, we now introduce a criterion for measuring robustness in probabilistic models: $r$-robustness. 
We start by presenting it in the general setting, before linking it to the specific case of VAEs.
\begin{definition}
\label{def:r_robust}
A model, $f$, operating on a point $\v{x}$, that outputs a continuous random variable is $r$-robust for $r \in \mathbb{R}^+$, to a perturbation $\v{\delta}$ and for an arbitrary norm $||\cdot||$ iff $$p(||f(\v{x} + \v{\delta}) - f(\v{x})|| \leq r) > p(||f(\v{x} + \v{\delta}) - f(\v{x})|| > r).$$ 
\end{definition}
We will assume from now on that the norm is taken to be the 2-norm $||\cdot||_2$, such that $r$-robustness determines a bound for which changes in the output $f(\v{x})$ induced by the perturbation $\v{\delta}$ are more likely to fall within the hyper-sphere of radius $r$, than not.
As $r$ decreases, the criterion for model robustness becomes stricter.
We note that $r$-robustness can be viewed as a probabilistic analog to the criterion for regression models presented by \cite{Nguyen2018}.
We also note that r-robustness can be generalized to $p(||f(\v{x} + \v{\delta}) - f(\v{x})||_2 \leq r) > m$, where $1-m$ is the allowable risk (with $m=0.5$ in Definition~\ref{def:r_robust}). 

Because this criterion is applicable to probabilistic models with continuous output spaces, it is directly relevant for ascertaining robustness in VAEs. 
By considering the smallest $r$ for which the criterion holds, we can think of it as a metric that provides a probabilistic measure of the \textit{extent} to which outputs are altered given a corrupted input:
the smaller the value of $r$ for which we can confirm $r$-robustness, the more robust the model.

\subsection{A Robustness Margin for VAEs}

We want to define a margin in a VAE's input space for which it is robust to perturbations of a given input. 
Perturbations that fall within this margin should not break our criterion for robustness.
Formally, we want a margin in $\mathcal{X}$, $R^r_{\mathcal{X}}(\v{x})$, for which any distorted input $\v x+\v{\delta}_x$, where $||\v{\delta}_x||_2 < R^r_{\mathcal{X}}(\v{x})$ is the perturbation, satisfies $r$-robustness when reconstructed. 

However, to consider the robustness of VAEs, we must not only take into account the perturbation $\v{\delta}_x$, but also the stochasticity of encoder.
We can think of the decoder as taking in noisy inputs because of this stochasticity.
Naturally, this noise can itself potentially cause issues in the robustness of VAE: if the level of noise is too high, we will not achieve reliable reconstructions even without perturbing the original inputs.
As such, before even considering perturbations, we first need to adapt our $r$-robustness framework to deal with this stochasticity.

\begin{figure}[t!]
    \centering
    \includegraphics[width=0.25\textwidth]{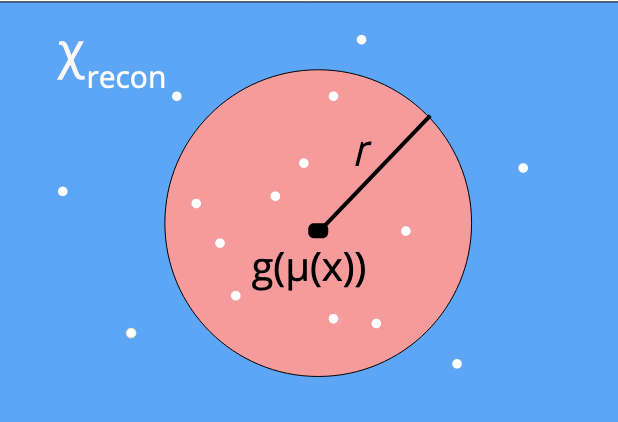}
    \vspace{-4pt}
    \caption{Illustration of $r$-robustness in a VAE.  
    White dots represent possible reconstructions, with the diversity originating from the encoder stochasticity.
    For $r$-robustness to hold, the probability of our reconstruction falling within the red area---a hypersphere of radius $r$ centered on $g_{\theta}(\v{\mu}_\phi(\v{x}))$---needs to be greater than or equal to the probability of falling outside.
    }
    \label{fig:r_robust}
\vspace{-6pt}
\end{figure}

\subsubsection{$r$-robustness for VAEs}

Given an input $\v{x}$, r-robustness dictates that we want to define some region in the reconstruction space, $\mathcal{X}_{\mathrm{recon}}$, within which most of the decoded samples from the latent embedding $\v{z}$ will fall.
We will assume that the encoder is a Gaussian as this is standard practice.
Denoting $g_{\theta}(\v z)$ as the deterministic mapping induced by the VAE's decoder network and $\v{\mu}_\phi(\v{x})$ as the mean embedding of the encoder, we can define 
$g_{\theta}(\v{\mu}_\phi(\v{x}))$ to be the ``maximum likelihood'' reconstruction, noting this is a deterministic function.
Our aim is now to find a hyper-sphere of radius $r$ centered on $g_{\theta}(\v{\mu}_\phi(\v{x}))$ within which most of the possible VAE outputs for a given point $\v{x}$ lie. 
Larger $r$ are indicative of a greater variance in the encoding process, and as such are likely to be associated with poorer quality reconstructions.

Denoting $\v{\eta} \sim \mathcal{N}(\v{0}, \v{I})$ as the reparameterized stochasticity of the encoder, we define the distance from the maximum likelihood reconstruction, induced by this sampling as 
\begin{align}
\Delta(\v{x}) = g_{\theta}(\v{\mu}_\phi(\v{x}) + \v{\eta}\circ \v{\sigma}_\phi(\v{x})) - g_{\theta}(\v{\mu}_\phi(\v{x})) 
\end{align}
Using this, we see that a VAE is $r$-robust to the stochasticity of the encoder iff (see also Fig~\ref{fig:r_robust})
\begin{align}
p(||\Delta(\v{x})||_2 \leq r) &> p(||\Delta(\v{x})||_2 > r).
\label{eq:robust_satisfy}
\end{align}
Informally, we want it to be more probable for reconstructions to fall within this radius $r$ than not.

\subsubsection{Robustness to distortions in data-space}

Given that we have established conditions for $r$ in Eq~\eqref{eq:robust_satisfy} that take into account latent space sampling, we can now return to our original objective, which was to determine a margin in the data-space $\mathcal{X}$ for which a VAE is robust to perturbations on its input.
Recall that this implicitly means that we want to define a bound for robustness given two sources of perturbations: the stochasticity of the encoder, and a hypothetical input perturbation $\v{\delta}_x$.

For simplicity of analysis, we consider the case where the perturbation is applied only to the encoder mean input and not the encoder variance input, noting that the latter is typically stable across inputs and so is less of a concern.
In Fig \ref{fig:musigattacks} we demonstrate that adversarial attacks on VAE encoders are dominated by the perturbation to the embedding mean that is induced, thereby justifying this assumption.
We note also that one can usually also simply fix the encoder variance to a constant for all datapoints without incurring substantial performance drops~\citep{Ghosh2020}, thereby providing a means to ensure this assumption holds exactly if needed.

\begin{figure}[t!]
    \centering
    \includegraphics[width=0.25\textwidth]{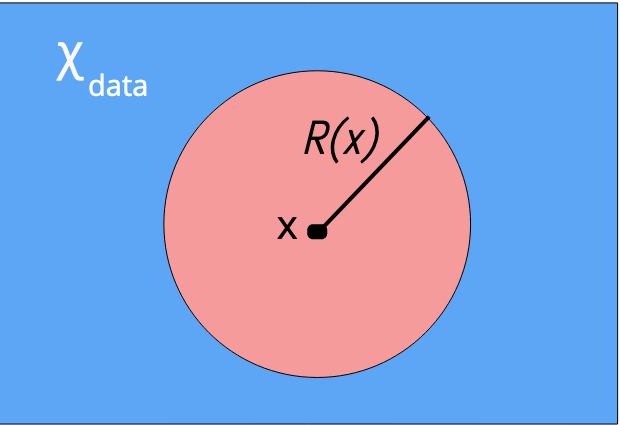}
    \vspace{-4pt}
    \caption{
    Illustration of the margin $R^r_{\mathcal{X}}(\v{x})$, which is defined in the \textbf{input} space $\mathcal{X}$. 
    Red represents represents the subspace where the model is $r$-robust, such that $p(||\Delta(\v{x}, \v{\delta}_x)||_2 \leq r) > p(||\Delta(\v{x},  \v{\delta}_x)||_2 > r)$ holds for all $\v x + \v{\delta}_x$ falling in this region, that is all $\v{\delta}_x : \lVert \v{\delta}_x \rVert_2 \le R^r_{\mathcal{X}}(\v{x})$.
    }
    \label{fig:boundary_in_X}
\vspace{-6pt}
\end{figure}

\begin{figure*}[t!]
    \centering
     \subfloat[][MNIST]{\includegraphics[width=0.24\textwidth]{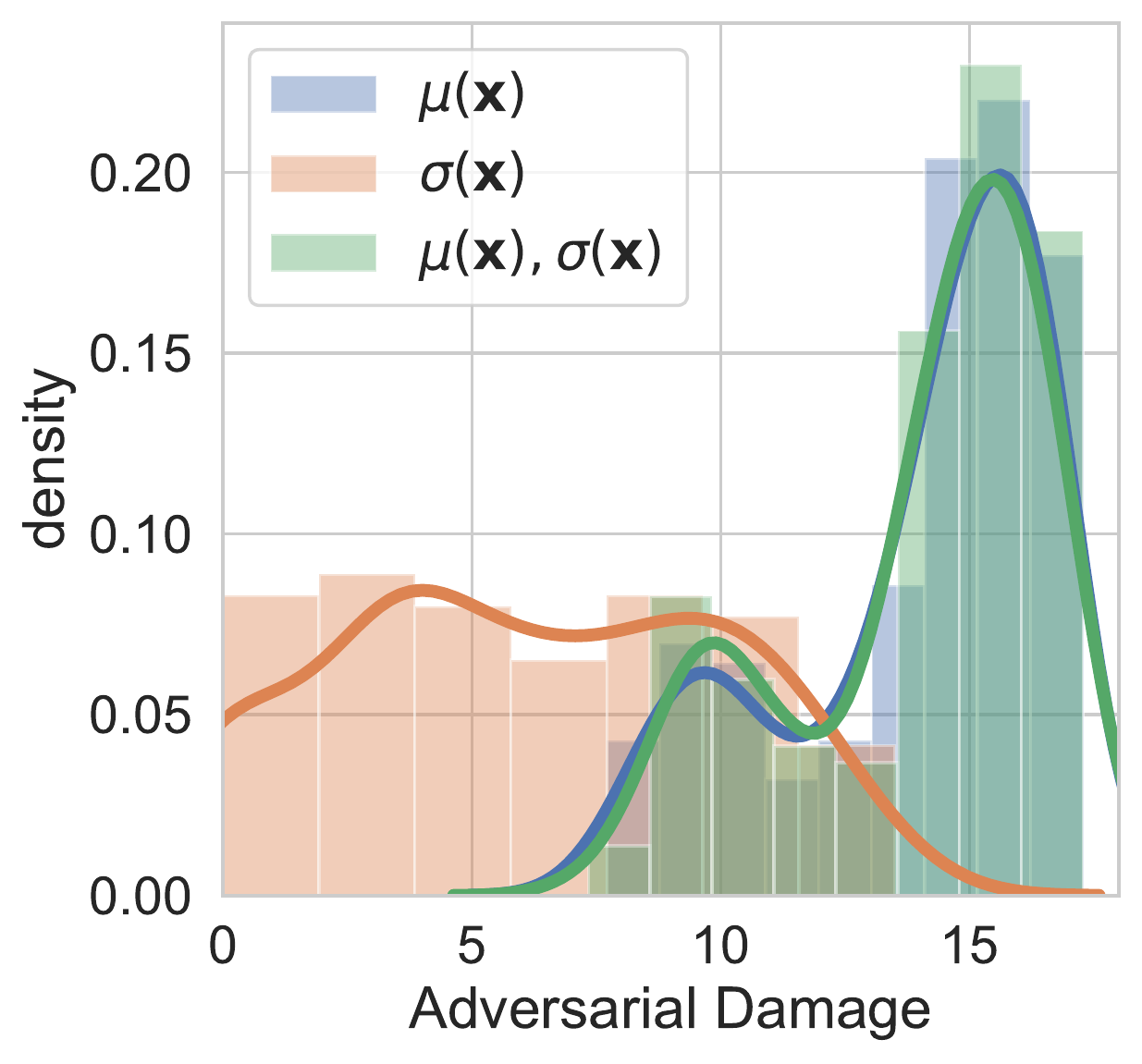}}
     \hspace{1.25em}
    \subfloat[][fMNIST]{\includegraphics[width=0.245\textwidth]{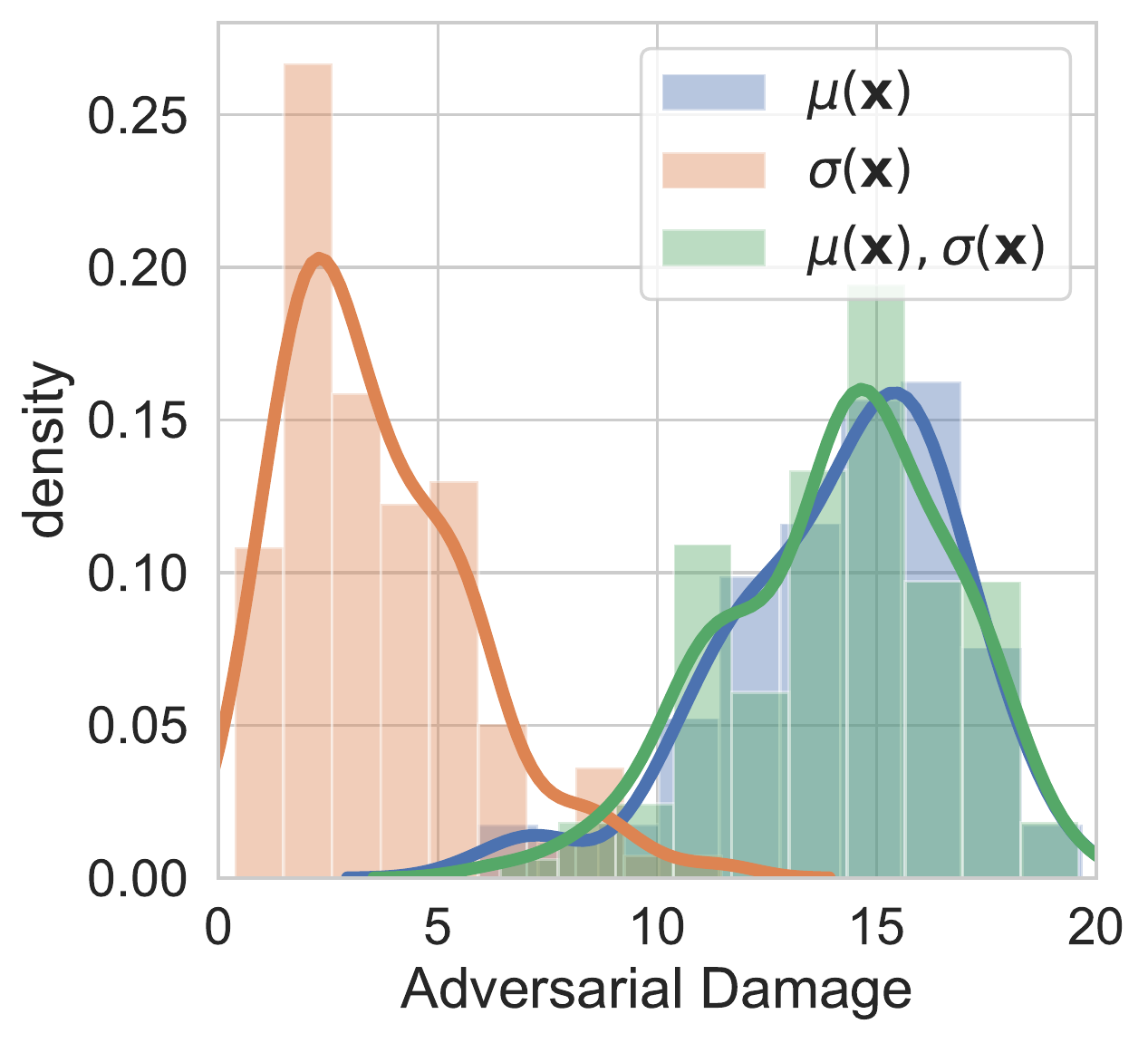}}
     \hspace{1.25em}
      \subfloat[][CIFAR10]{\includegraphics[width=0.235\textwidth]{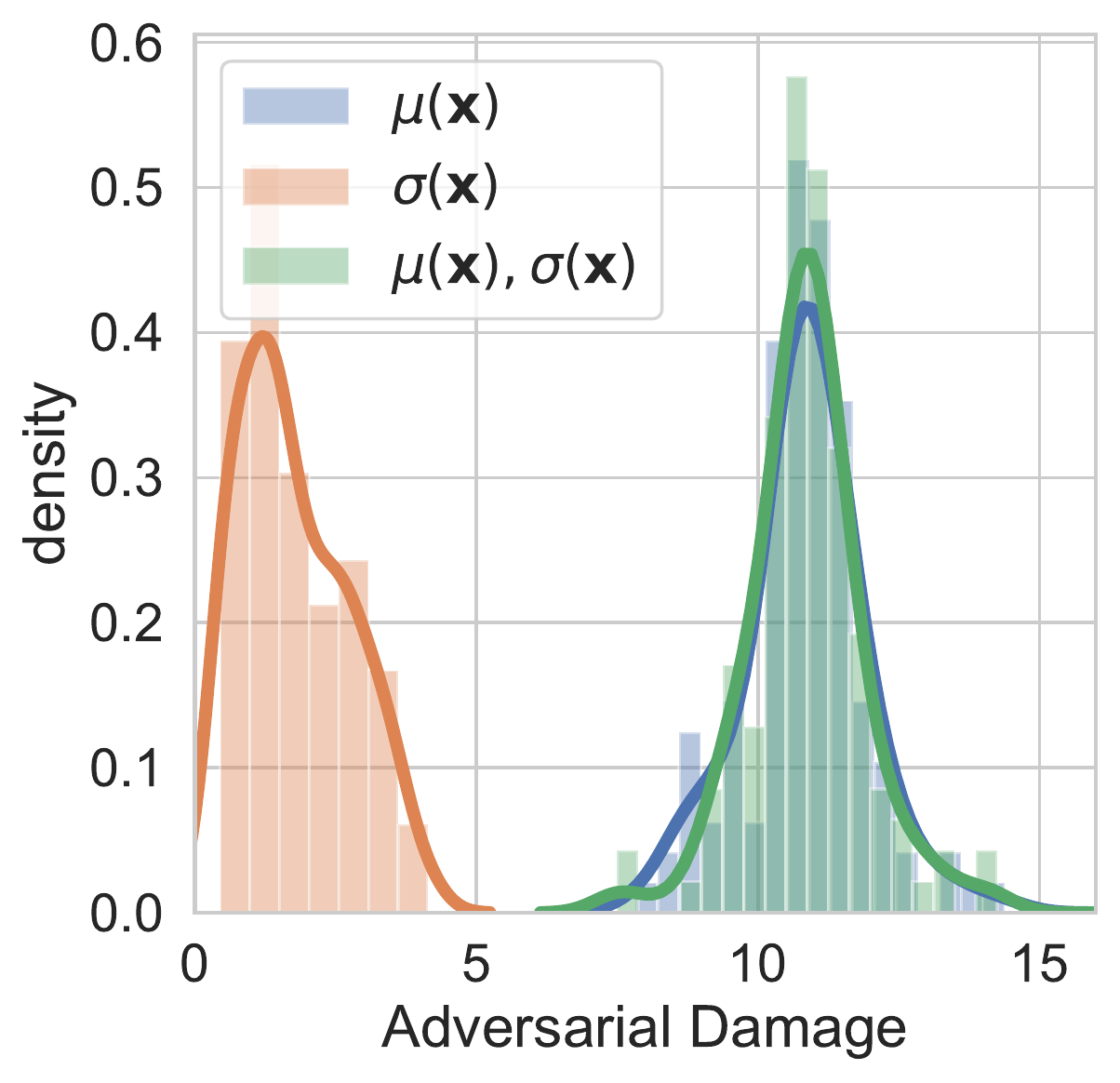}}
      \vspace{-6pt}
    \caption{
    \textit{Maximum damage} adversarial attacks (see Eq \eqref{eq:adv_loss}) on multiple VAEs trained on MNIST (a), fashion-MNIST (b), and CIFAR10 (c). 
    We attack 25 datapoints for each VAE and propagate the attacks to the encoder mean ($\mu(\v{x})$), the encoder standard deviation ($\sigma(\v{x})$), or both ($\mu(\v{x}), \sigma(\v{x})$).
    Attack norms are capped to 10.
    Shown are distribution plots of the adversarial damage, i.e.~the $L_2$ distance between the reconstruction resulting from the attack and the maximum likelihood reconstruction $g_{\theta}(\v{\mu}_\phi(\v{x}))$. 
    Clearly attacks on $\mu(\v{x})$ are more harmful than on $\sigma(\v{x})$, and most of the damage from attacks on both $\mu(\v{x})$ and $ \sigma(\v{x})$ stems from the attack on $\mu(\v{x})$. 
    }
    \label{fig:musigattacks}
    \vspace{-6pt}
\end{figure*}

We define the distance from the maximum likelihood reconstruction, $g_{\theta}(\v{\mu}_\phi(\v{x}))$, induced by the stochasticity of the encoder \textit{and} an input perturbation $\v{\delta}_x$ as
\begin{align}
\Delta(\v{x}, \v{\delta}_x) = g_{\theta}(\v{\mu}_\phi(\v{x} + \v{\delta}_x) + \v{\eta}\sigma_\phi(\v{x})) - g_{\theta}(\v{\mu}_\phi(\v{x})).  \nonumber
\end{align}
We can now define the condition for which $r$-robustness is satisfied on the VAE output given the two sources of perturbation as
\begin{align}
    &||\v{\delta}_x||_2 < R^r_\mathcal{X}(\v{x}) ~ \Leftrightarrow ~ p(||\Delta(\v{x}, \v{\delta}_x)||_2 \leq r) > 0.5 
\end{align}
Thus, $R^r_\mathcal{X}(\v{x})$ is the margin of robustness of the VAE such that $\forall \v{\delta}_x \,:\, ||\v{\delta}_x||_2<R^r_\mathcal{X}(\v{x})$, $\v{x} + \v{\delta}_x$ is more likely than not to be reconstructed within a radius $r$ of the maximum likelihood reconstruction $g_\theta(\v{\mu}_{\phi}(\v{x}))$. 
A high level illustration of this is given in Fig \ref{fig:boundary_in_X}, and Fig \ref{fig:R_demo} shows a simple empirical demonstration of how $R^r_\mathcal{X}(\v{x})$ relates to the probability of producing a good reconstruction under random input perturbations.

We note that, analogously to the concept of an adversarial risk~\citep{uesato2018adversarial}, $R^r_{\mathcal{X}}(\v x)$ can further be converted to a metric for the \emph{overall} robustness of a VAE, by taking its expectation over the data generating distribution, namely
    $R^r_{\mathcal{X}} = \expect_{p_{\mathcal{D}}(\v x)} \left[R^r_\mathcal{X}(\v{x})\right]$.

\subsection{Characterizing the Margin}
Given this definition, we now wish to try and characterize $R^r_\mathcal{X}(\v{x})$.
In particular, we would like to understand what characteristics of the VAE are likely to make it relatively larger or smaller.
Ideally, we also want to establish scenarios where we might be able to provide guarantees of a minimum size for $R^r_\mathcal{X}(\v{x})$, such that we might be able to make inroads into how one might apriori construct a certifiably-robust VAE.

A perturbation in $\mathcal{X}$, $\v{\delta}_x$, induces a perturbation in $\mathcal{Z}$, $\v{\delta}_z$. To determine the margins for robustness in $\mathcal{X}$, we first apply the Neyman-Pearson lemma \citep{Neyman1992, Cohen2019}, assuming a ``worst-case'' decoder. This decoder has subspaces in  $\mathcal{Z}$, where it is is either robust or non-robust, that are divided by a boundary that is normal to both the induced perturbation $\v{\delta}_z$ and to the dimension of minimal variance in $\mathcal{Z}$, $\min_i \v{\sigma}_\phi(\v{x})_i$.
We then determine the minimum perturbation norm in $\mathcal{X}$ which induces a perturbation in $\mathcal{Z}$ that crosses this boundary.
\begin{restatable}{theorem}{Rbound}
Consider a VAE with a diagonal-variance Gaussian encoder, an input $\v x$, and an output margin $r\in \mathbb{R}$ such that the VAE is $r$-robust to the stochasticity of the encoder when the $\v x$ is unperturbed as per~\eqref{eq:robust_satisfy}.
Assuming standard regularity assumptions (discussed in the proof) hold for $\v{\mu}_{\phi}(\v x)$, then 
\begin{align}
R^r_\mathcal{X}(\v{x}) \geq \frac{(\min_i \v{\sigma}_\phi(\v{x})_i)\Phi^{-1}(p(||\Delta(\v{x})||_2 \leq r))}{||\v{J}^{\mu}_{\phi}(\v{x})||_F } + \mathcal{O} (\v \varepsilon)
\label{eq:margin_robust_vae_data_lower_bound}
\end{align}
where $\mathcal{O} (\v \varepsilon)$ represents higher order dominated terms that disappear in the limit of small perturbations,
$\Phi^{-1}$ is the probit function, %
$\v{J}^{\mu}_{\phi}(\v{x})_{i,j}=
\partial \v{\mu}_\phi(\v{x})_{i} / \partial \v{x}_j$ is the Jacobian of $\v{\mu}_{\phi}(\v x)$, and $||\cdot||_F$ is the Frobenius norm.
\label{prop:margin_x}
\end{restatable}
The proof is provided in Appendix \ref{app:margin_X}.
This bound is based on a first order approximation of $\v{\mu}_{\phi}(\v x+\v \delta_x)$ around the original input $\v x$; the impact of $\mathcal{O} (\v \varepsilon)$ thus depends on how well this approximation holds.
As such, the result is particularly applicable to networks with piecewise linear activation functions such as the $\mathrm{ReLU}$, which are locally linear and are among the most widely used activation functions.
For these activation functions this bound is locally exact: $\mathcal{O} (\v \varepsilon)$ is exactly zero if the size of the bound is smaller than what is required to go outside the locally linear region.

\begin{figure}[t!]
    \centering
     \includegraphics[width=0.25\textwidth]{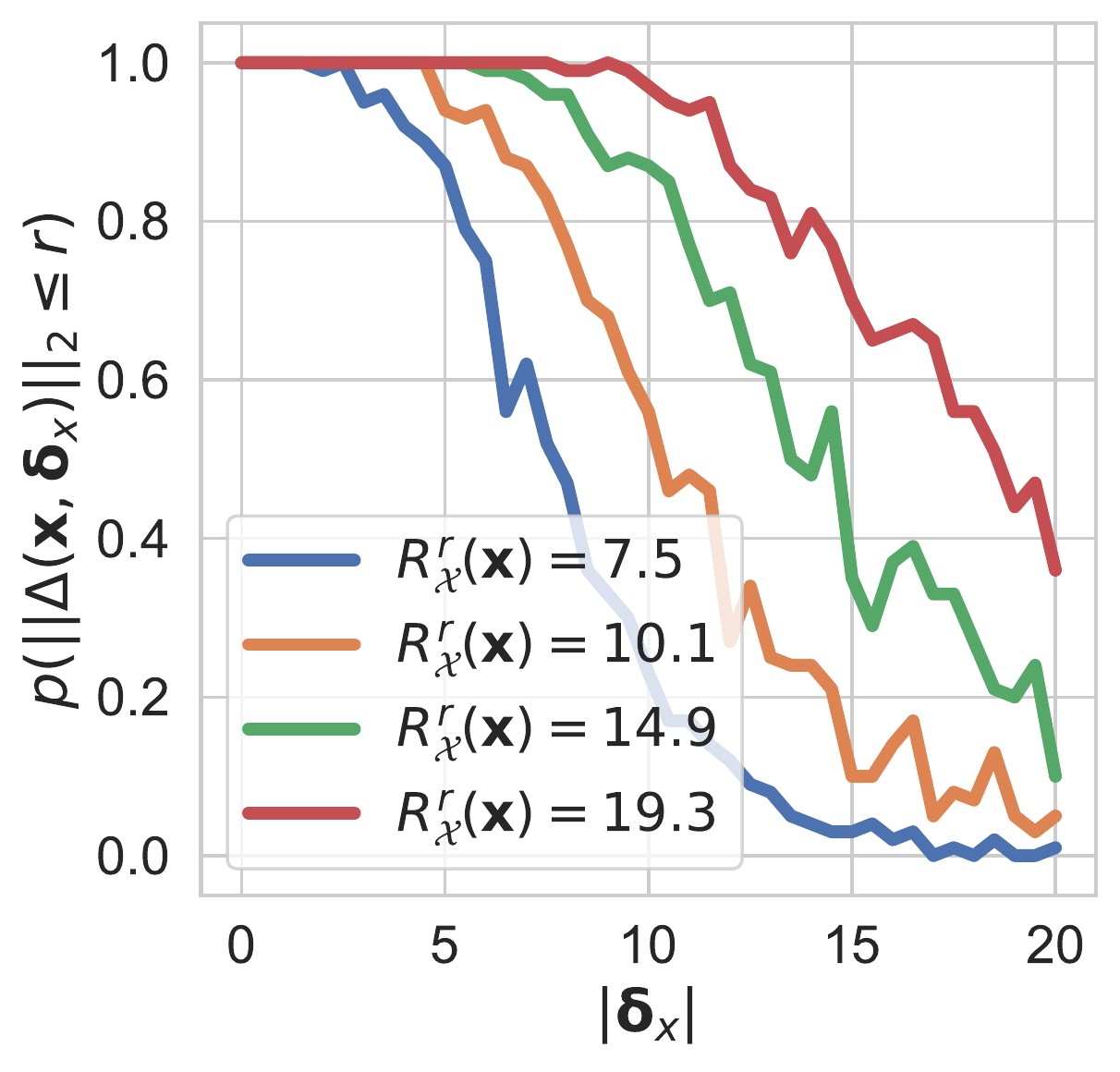}
    \vspace{-4pt}
    \caption{$R^r_\mathcal{X}(\v{x})$ for four VAEs of varying robustness trained on MNIST. 
    We fix the input $\v x$ and perturbation direction $\v{\delta}_x/\lVert\v{\delta}_x\rVert_2$, but vary the perturbation size $\lVert\v{\delta}_x\rVert_2$. We assess the proportion of samples which fall within $r\!=\!4$ of the maximum likelihood reconstruction.}
    \label{fig:R_demo}
    \vspace{-10pt}
\end{figure}

\begin{figure*}[t!]
         \centering
     \subfloat[][MNIST]{\includegraphics[width=0.257\textwidth]{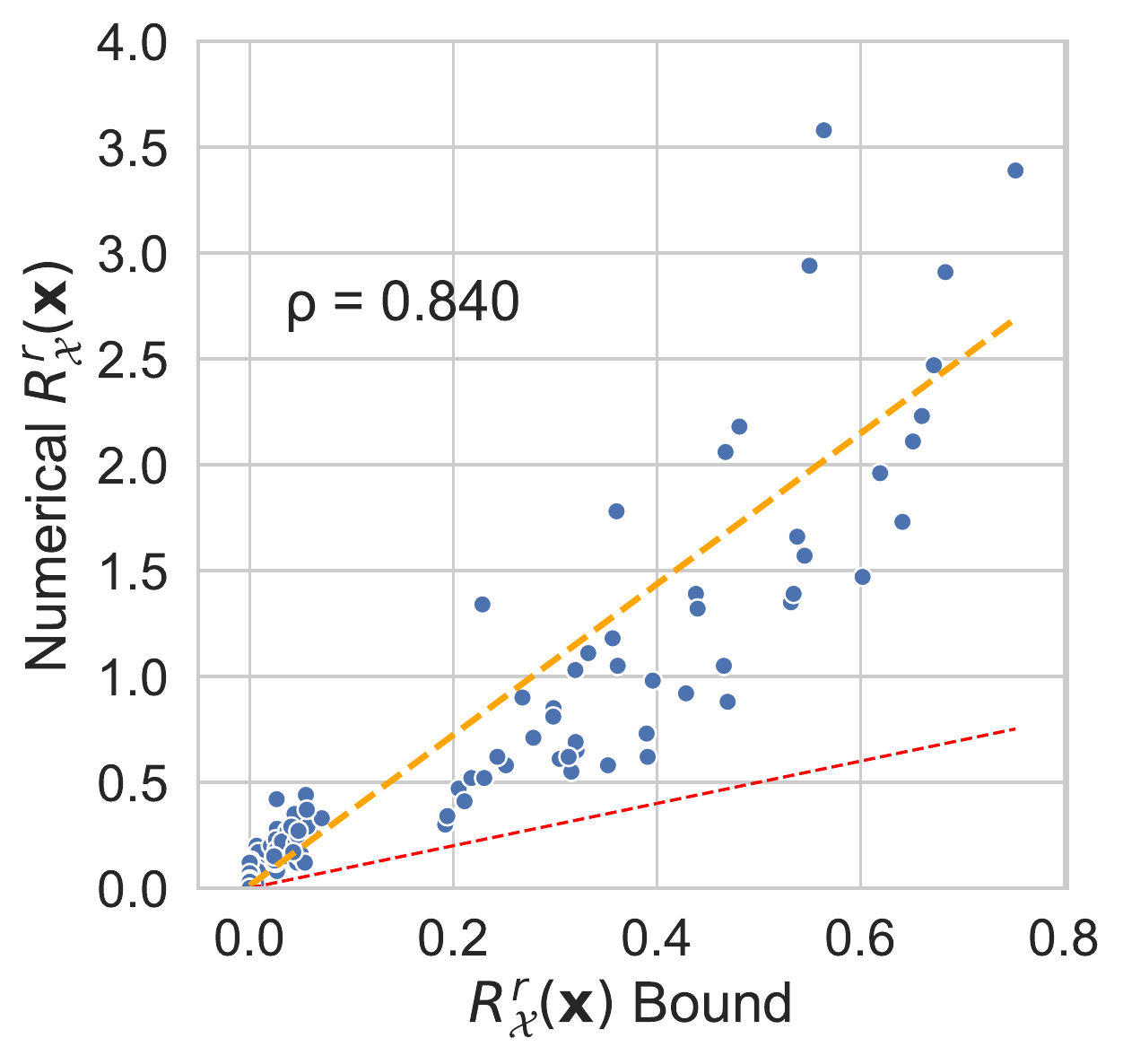}}
      \hspace{1em}
      \subfloat[][fMNIST]{\includegraphics[width=0.25\textwidth]{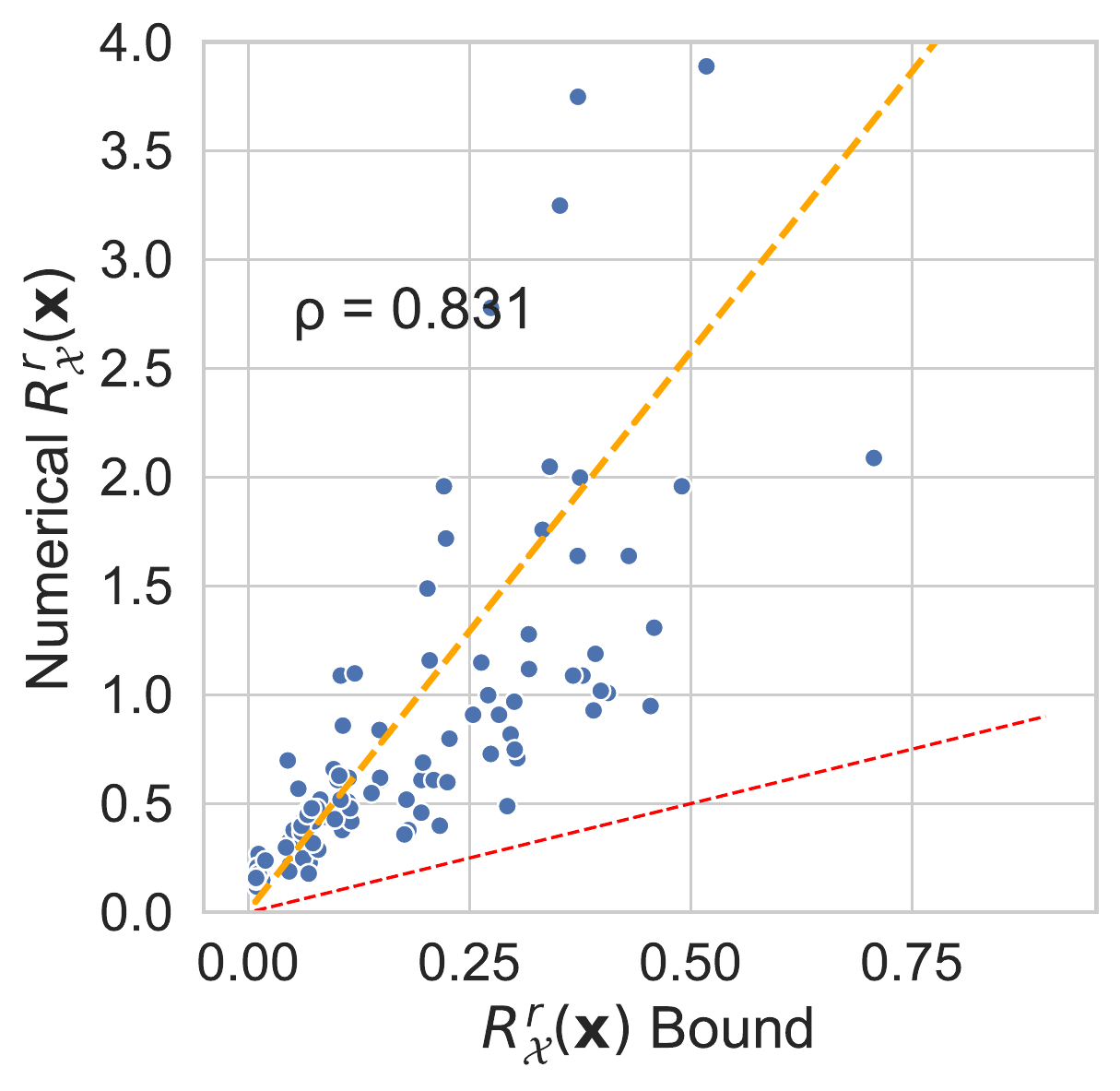}}
       \hspace{1em}
      \subfloat[][CIFAR10]{\includegraphics[width=0.25\textwidth]{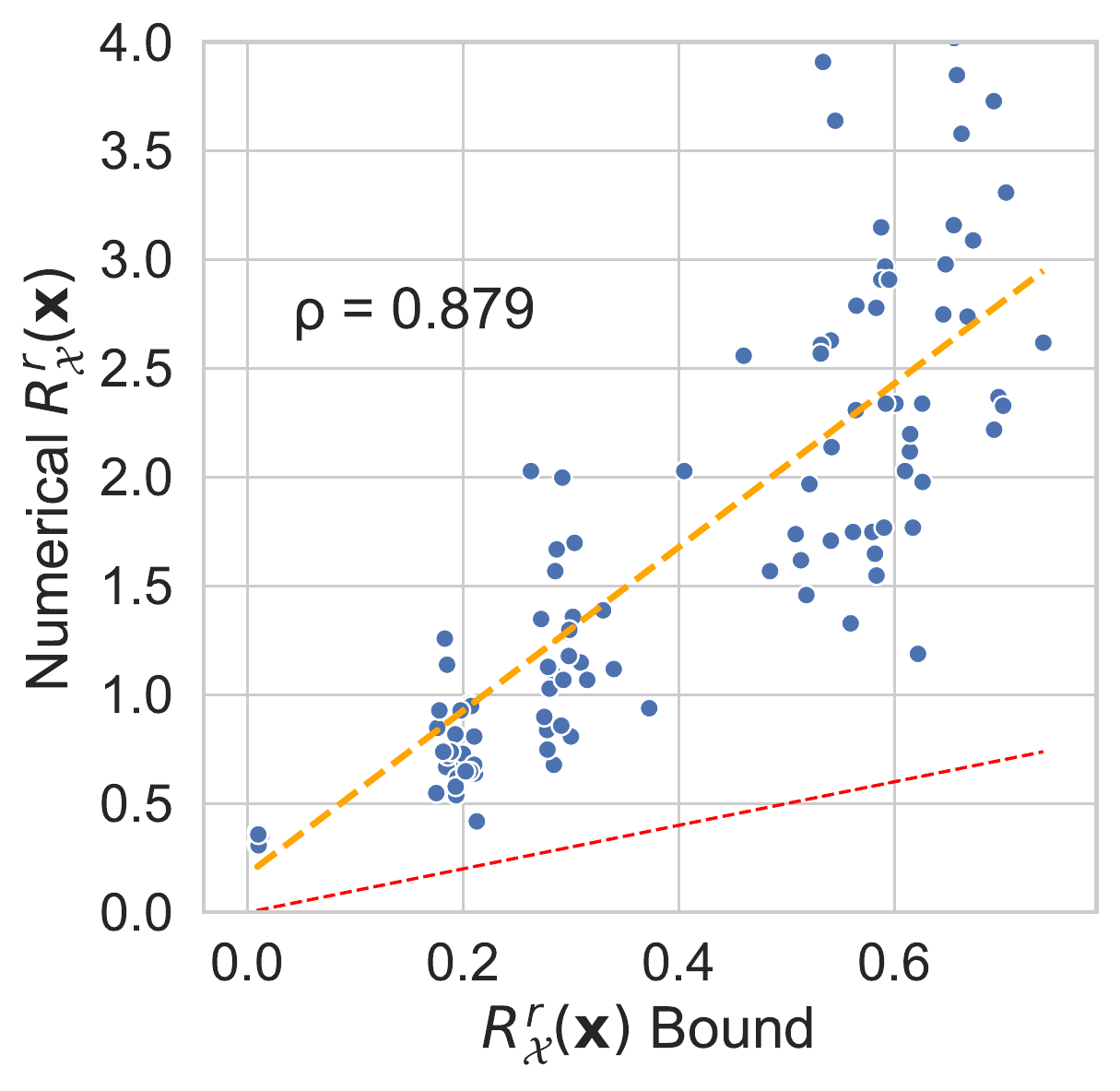}}\\
      \vspace{-4pt}
    \subfloat[][MNIST]{\includegraphics[width=0.25\textwidth]{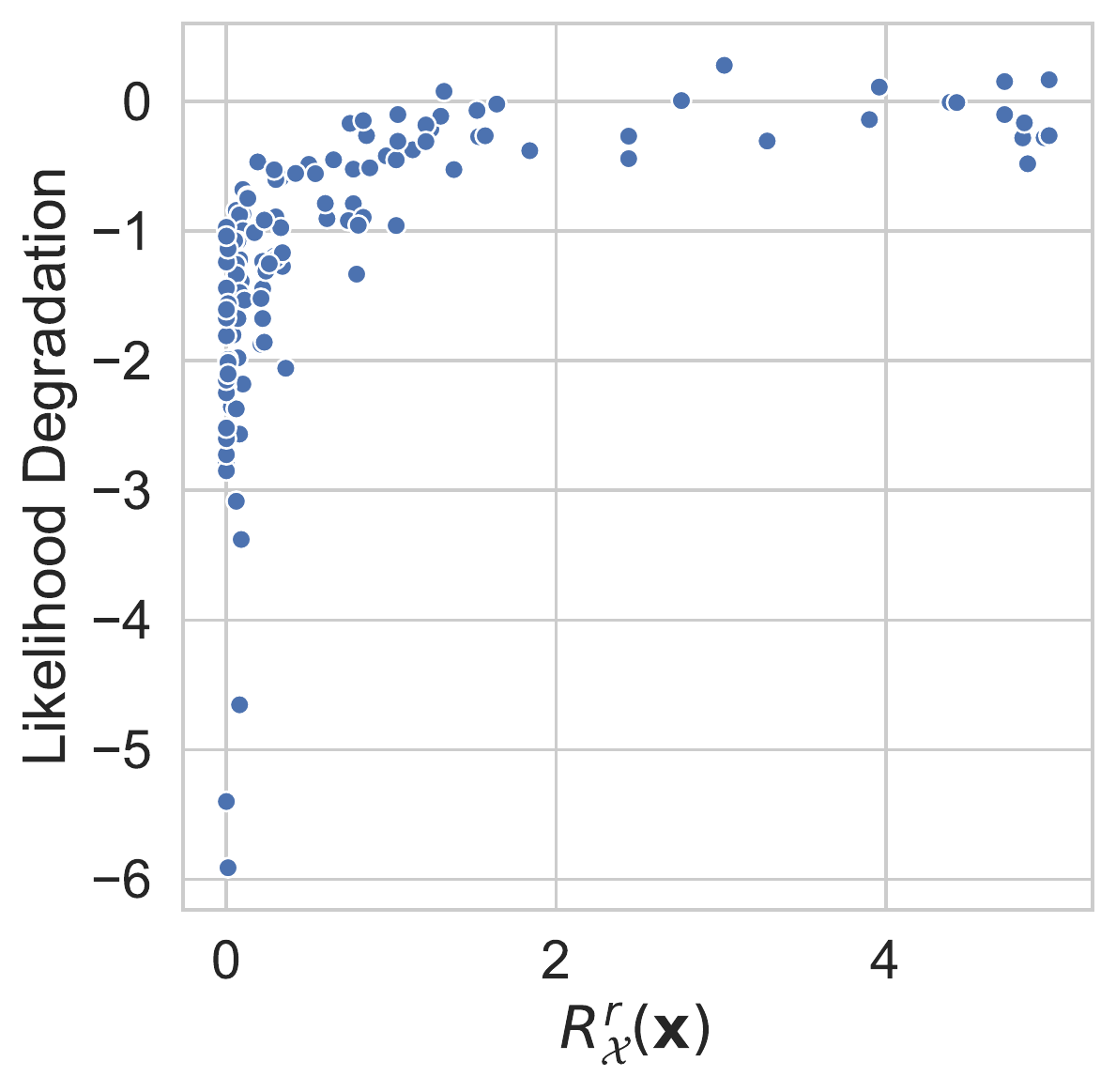}}
    \hspace{1em}
    \subfloat[][fMNIST]{\includegraphics[width=0.25\textwidth]{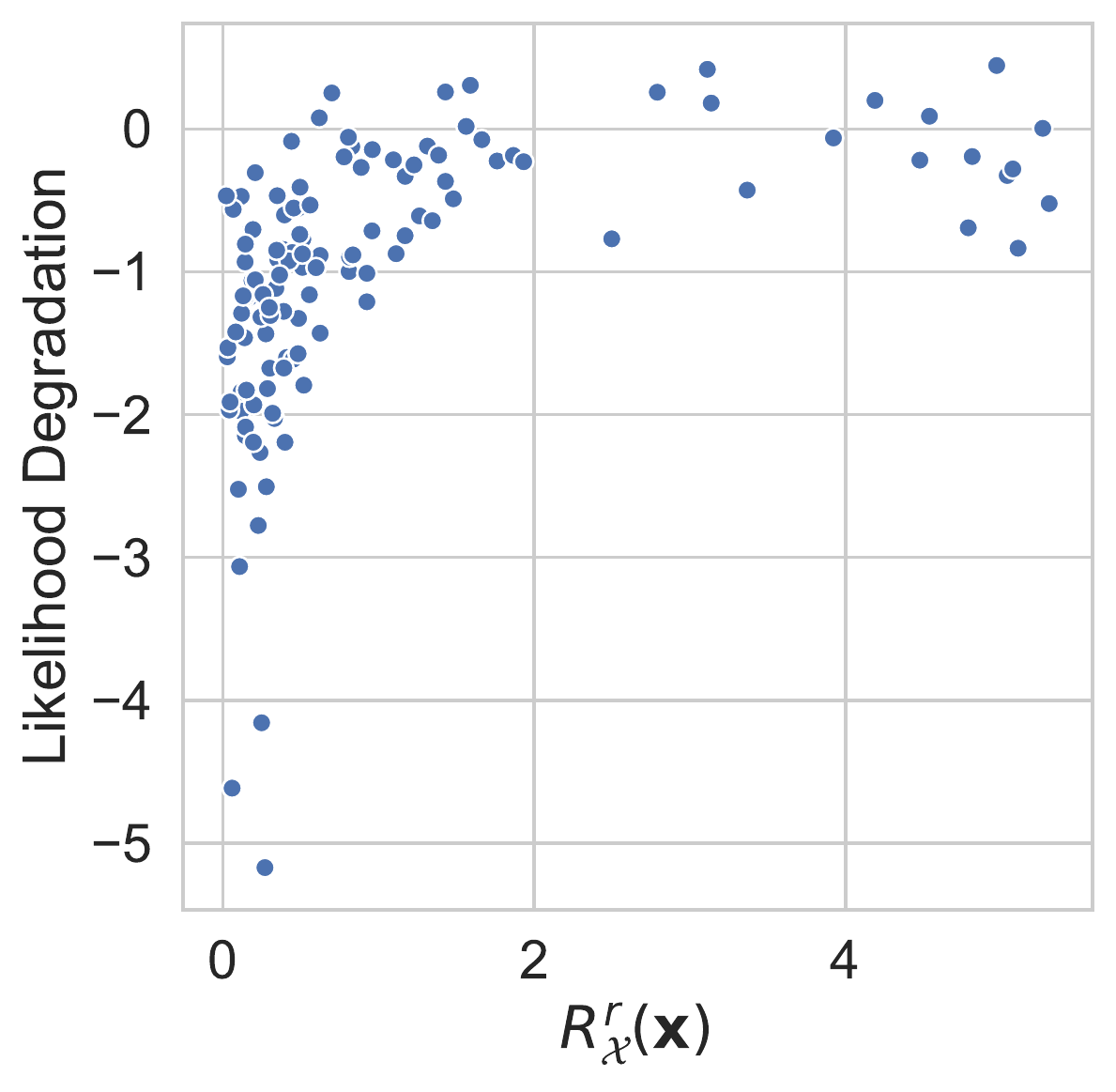}}
     \hspace{1em}
     \subfloat[][CIFAR10]{\includegraphics[width=0.26\textwidth]{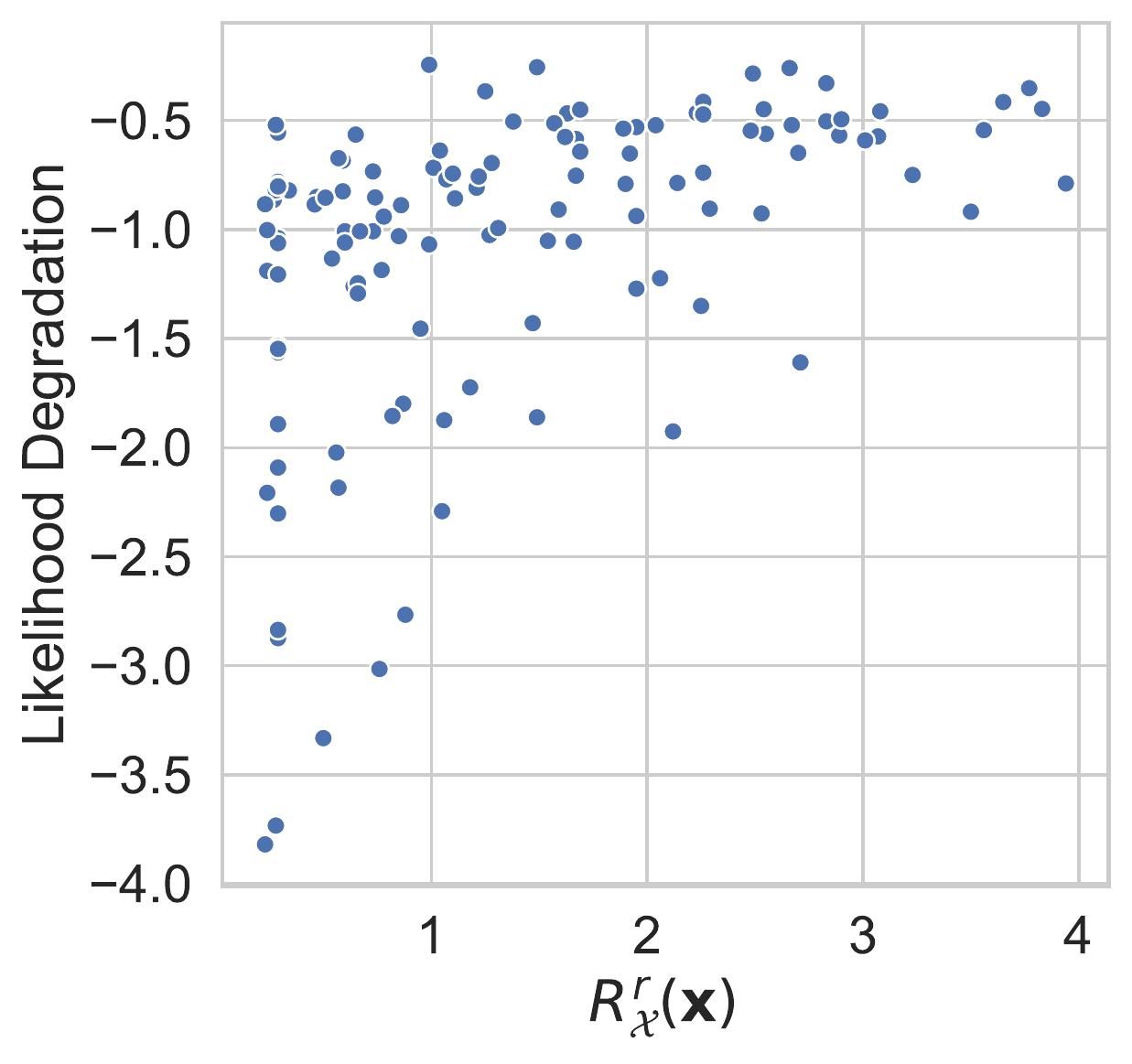}}
    \vspace{-4pt}
    \caption{
    (a-c) show the empirically estimated $R^r_{\mathcal{X}}(\v{x})$ against the bound for $R^r_{\mathcal{X}}(\v{x})$ defined in Theorem 1, ignoring higher order terms. 
    Each dot represents a network--input pair, with $5$ separately trained networks and $25$ distinct inputs considered.
    We show the line of best fit (in orange), the correlation coefficient $\rho$, and the line $y=x$ (in red) representing the theoretical bound itself.
    (d-f) show the relative log likelihood degradation resulting from a `maximum-damage' adversarial attack against the numerically estimated $R^r_{\mathcal{X}}(\v{x})$ for these same VAEs and inputs (see Section~\ref{sec:rrobustadv}).
    }
    \vspace{-8pt}
    \label{fig:R_robustness_correlation}
\end{figure*}

This gives us margins for which VAEs are certifiably robust, up to first order expansions, to adversarial perturbations on their inputs; they have similar forms to the sensitivity margins for classifiers defined by \cite{Sokolic2017,Jakubovitz2018} in that both scale \textit{inversely} with the network Jacobian.
More generally, these results provide insights into the features which lead to robust VAEs.
As shown in Figure \ref{fig:R_robustness_correlation}, the bound seems to be relatively tight in practice, even when attacking both the encoder mean \textit{and} variance.  It also has a near-linear relationship with the empirically estimated robustness, such that it forms a powerful and convenient robustness metric in its own right.

Examining the bound, we see that for a given $r$, 
$R^r_\mathcal{X}(\v{x})$ increases as the stochasticity of the encoder, i.e $\v{\sigma}_\phi(\v{x})$, increases, provided that this does not overly affect $\Phi^{-1}(p(||\Delta(\v{x})||_2 \leq r))$ (see below).
As $\v \sigma_\phi(\v{x})$ tends to 0 we recover the deterministic setting, which confers no additional protection to attack and as $\v \sigma_\phi(\v{x})$ increases we obtain increased protection.
 However, $\v{\sigma}_\phi(\v{x})$ can also have a knock--on effect on $\Phi^{-1}(p(||\Delta(\v{x})||_2 \leq r))$.
When $\v{\sigma}_\phi(\v{x})$ is small, this knock--on effect will typically be small relative to the direct effect of changing $\v{\sigma}_\phi(\v{x})$, but as it becomes large there is always a point where this knock--on effect will take over.
Namely, our reconstructions will become increasingly poor and $\Phi^{-1}(p(||\Delta(\v{x})||_2 \leq r))$ will eventually become negative, such that $r$-robustness does not hold even without perturbation.
We can quantify this by noting that there is always a minimum $r$ for $r$-robustness to be satisfied in~\eqref{eq:robust_satisfy}.
We derive a bound characterizing the minimum $r$ for which we can confirm robustness in Appendix~\ref{app:min_r}.

\section{Empirical Investigations}

\begin{figure*}[h]
         \centering
      {\includegraphics[width=0.25\textwidth]{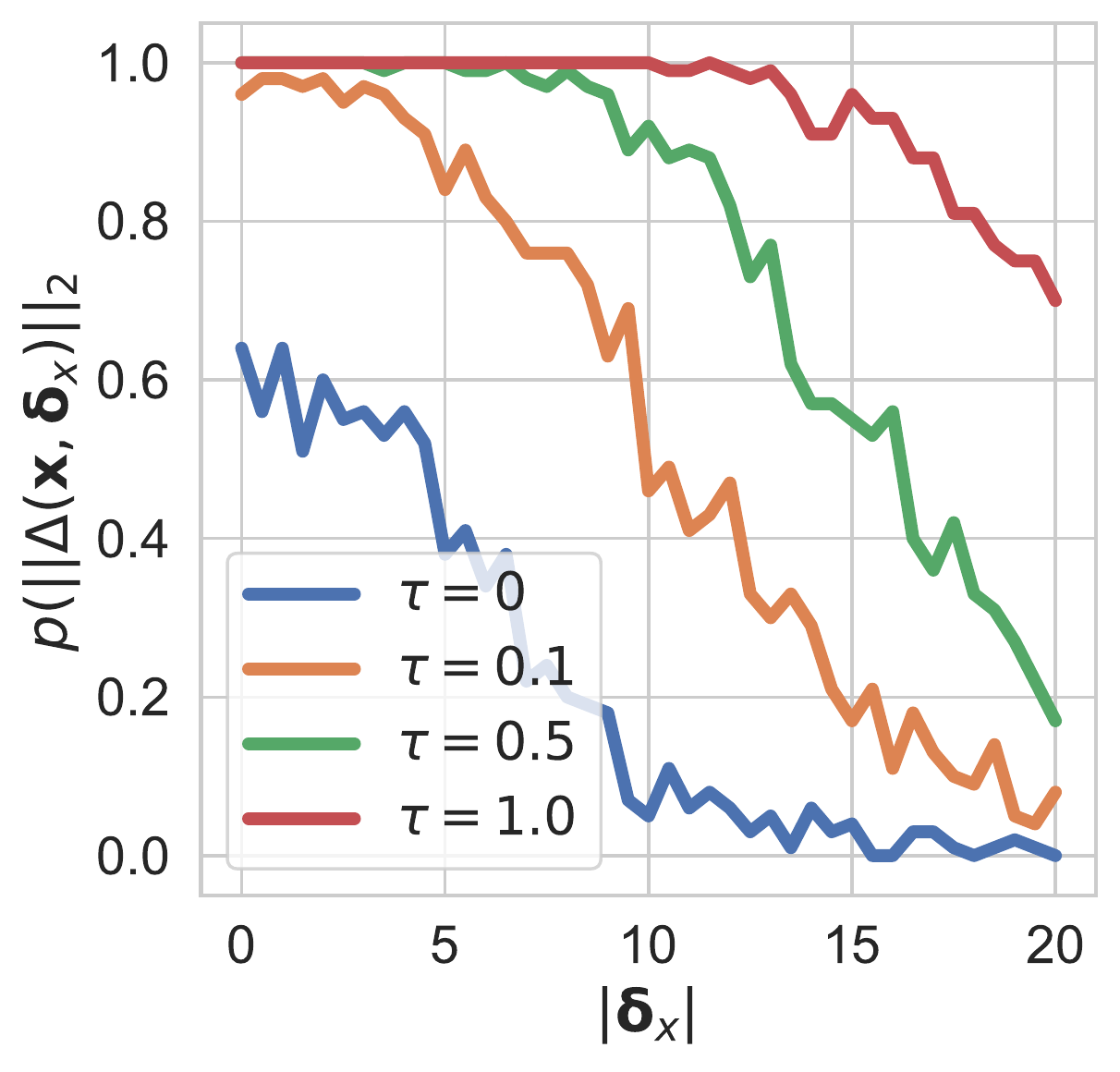}}
           \hspace{1.25em}
     {\includegraphics[width=0.27\textwidth]{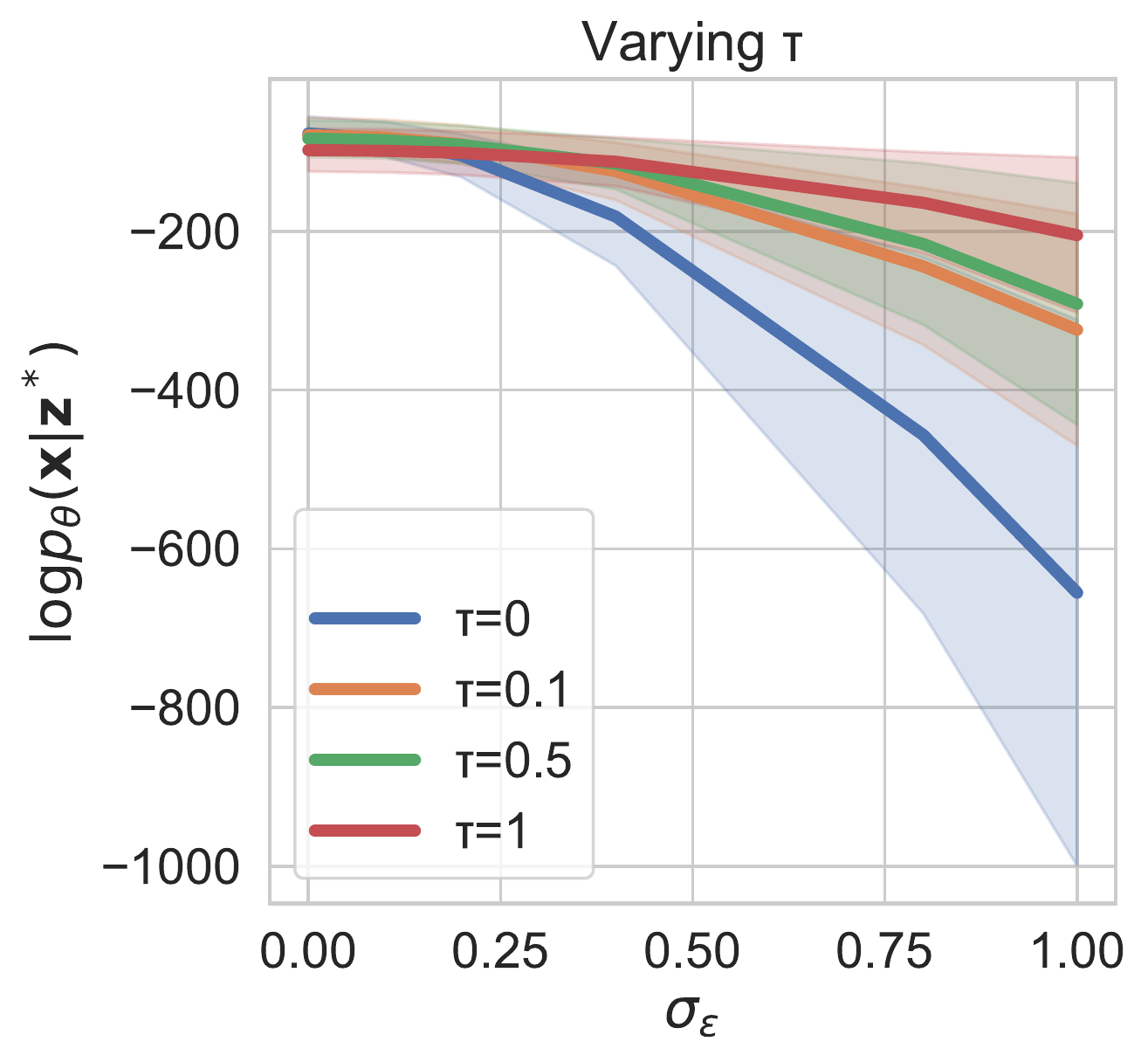}} 
          \hspace{1.25em}
    {\includegraphics[width=0.27\textwidth]{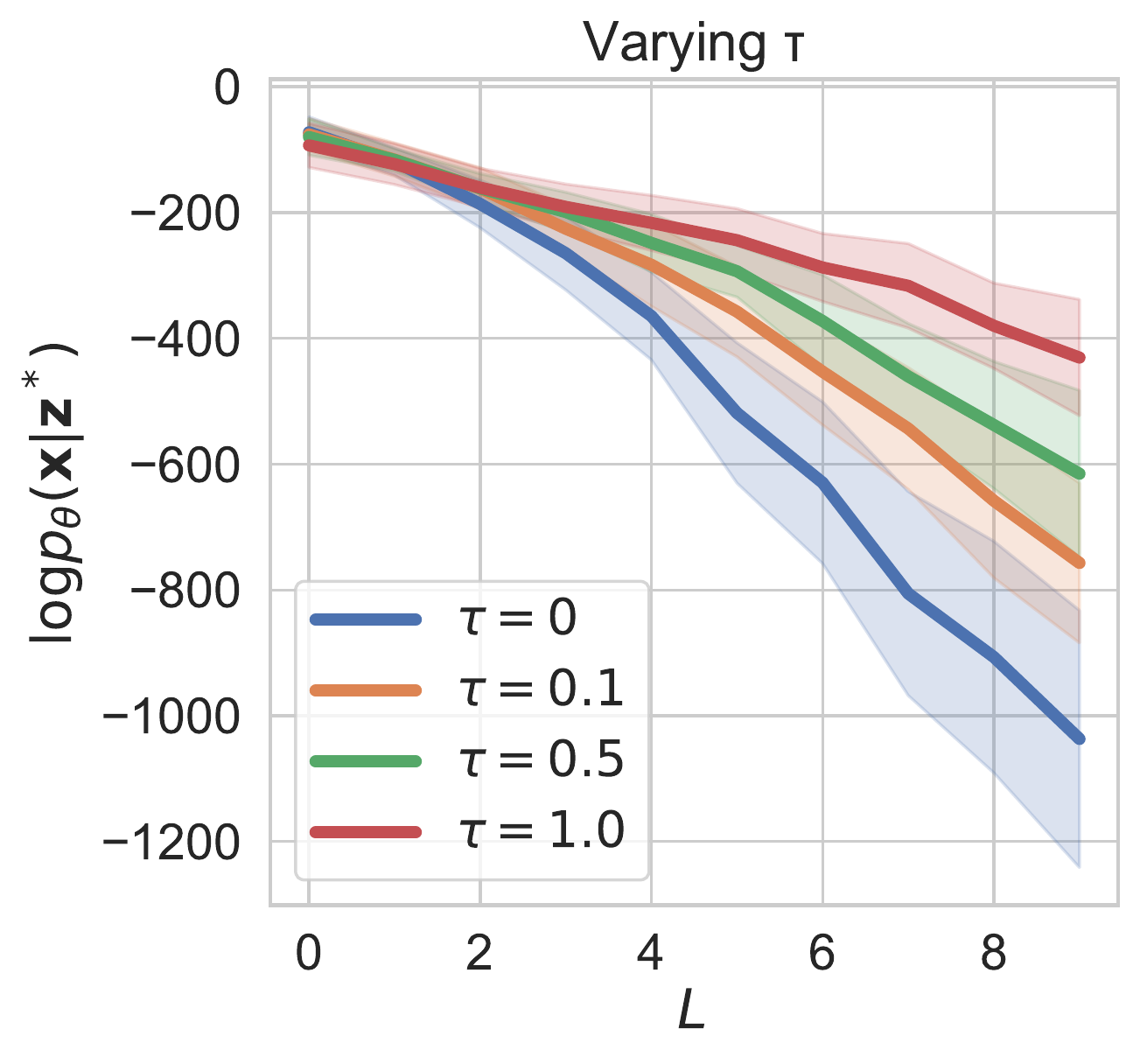}}
    \caption{Ablation study on the bounds defined by Theorem \ref{prop:margin_x}.
    We train models  on MNIST with $\sigma_\phi(\v{x})$ offset by a constant $\tau \in [0, 0.1, 0.5, 1]$.
    [Left] probability that reconstructions in $\mathcal{X}_{\mathrm{recon}}$ fall within a radius $r=4$ centered on the `maximum likelihood' reconstruction, $p(||\Delta(\v{x}, \v{\delta}_x)||_2 \leq r)$, as a function of $|\v{\delta}_x|$, the magnitude of perturbations. 
    $R^r_\mathcal{X}(\v{x})$ is the radius $|\v{\delta}_x|$ for which $p(||\Delta(\v{x}, \v{\delta}_x)||_2 \leq 4) > 0.5$ and clearly increases with $\tau$.
    [Center]  we add noise $\sim\mathcal{N}(0,\sigma^2_{\epsilon})$ to a point $\v{x}$ forming a noisy $\v{x}^*$ and $\v{z}^*$, and measure the likelihood of the original point $\v{x}$ under this noisy embedding.
    [Right] we show the same plot where the perturbations are \textit{maximum damage} attacks, Eq \eqref{eq:adv_loss}, where
    $L$ is the maximum allowed magnitude of the attack distortion. 
    Large $\tau$ VAEs have high likelihoods for the original point $\v{x}$ as $L$ and $\sigma^2_{\epsilon}$ increase: they are robust to attack and effective denoising models.
    Confidence intervals are the standard deviations of values over the entire MNIST dataset.
    }
     \label{fig:ablation}
\vspace{-0.11cm}
\end{figure*}

We now consider a series of empirical investigations to back up our frameworks and theoretical results.
We start by assessing whether the concept of $r$-robustness corresponds to more commonly used measures of model robustness.
Here we estimate $R^r_{\mathcal{X}}(\v{x})$ numerically as in Appendix \ref{app:emp_calc} and as demonstrated in Fig \ref{fig:R_demo}.
Using these empirical estimations, we establish connections between $r$-robustness and other performance metrics during adversarial attack, confirming that larger $R^r_{\mathcal{X}}(\v{x})$ correspond to model--input pairs that are more robust to adversarial attacks. 

\subsection{$r$-robustness and Adversarial Settings}
\label{sec:rrobustadv}
We begin by evaluating our metrics in adversarial settings. 
We want to find the most damaging perturbations $\v{\delta}_x$ that challenge the robustness metrics we have derived.
We consider an adversary trying to distort the input data to maximally disrupt a VAE's reconstruction.
Our adversary maximizes, wrt $\v{\delta}_x$, the distance between the VAE reconstruction and the original datapoint $\v{x}$, a novel adversarial attack we call \textit{maximum damage}. 
We attack the encoder mean \textit{and} variance: 
\begin{align}
    \v{\delta}_x^* =\argmax\nolimits_{\v{\delta}_x}\big(&||g_{\theta}(\v{\mu}_\phi(\v{x} + \v{\delta}_x)+ \v{\eta}\sigma_\phi(\v{x} + \v{\delta}_x)) \nonumber\\
    &- g_{\theta}(\v{\mu}_\phi(\v{x}))||_2\big).
    \label{eq:adv_loss}
\end{align}
We evaluate the success of an attack as follows.
Given an embedding $\v{z}^*$ formed from the mean encoding of $\v{x} + \v{\delta}_x$, we measure the likelihood of the original point $\v{x}$ and quantify the degradation in model performance as the relative log likelihood degradation ($|(\log p(\v{x}|\v{z}^*) - \log p(\v{x}|\v{z}))|/\log p(\v{x}|\v{z})$), where $\v{z}$ is the embedding of $\v{x}$. 
Fig \ref{fig:R_robustness_correlation} (d-f) shows that as $R^r_{\mathcal{X}}(\v{x})$ increases this degradation lessens, indicating less damaging attacks.
As such, larger margins for $r$-robustness correspond to models that are more robust to attack.

\subsection{Evaluating the derived bounds}
Using Theorem \ref{prop:margin_x}, we can gain insights into which characteristics of a VAE contribute to robustness.
The encoder variance plays a prominent role and is a parameter that is easy to control. 
The encoder Jacobian is also present, but we found that controlling such values directly can be difficult.
Penalizing the norm of this Jacobian in the VAE training objective degrades VAE generative performance, making it difficult to compare models. %
As such we restrict our experiments to varying the encoder variance. 
We do so by training models that have $\v{\sigma}_\phi(\v{x})$ offset by a constant $\tau$, such that we artificially increase the encoder variance minimum.
In Fig \ref{fig:ablation} we show that as $\tau$ increases, the numerically estimated $R^r_\mathcal{X}(\v{x})$ also increases, supporting our claim that models with larger encoder variances have larger margins of robustness.
This figure also shows that likelihood of reconstructing the original input $x$ increases as $\tau$ increases in both an adversarial attack setting and a noisy perturbation setting.
We thus see that larger $\tau$ also provides more effective denoising properties.

\section{Robustness of Disentangled VAEs}

\begin{figure*}[h!]
    \centering
    \subfloat[][$\beta=0.1$]{
    \begin{tabular}[b]{c}%
    \includegraphics[width=0.25\textwidth]{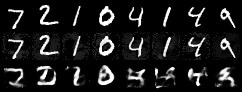}
    \end{tabular}
    } 
    \hspace{0mm}
     \subfloat[][$\beta=1$]{
     \begin{tabular}[b]{c}%
     \includegraphics[width=0.25\textwidth]{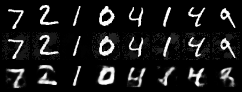}
      \end{tabular}
     } 
     \subfloat[][$\beta=10$]{
      \begin{tabular}[b]{c}
     \includegraphics[width=0.25\textwidth]{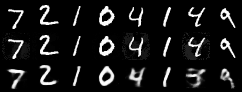}
      \end{tabular}
  } 

    \caption{
    For $\beta$-VAEs trained with $\beta \in \{0.1,1,10\}$ we show in consequentive rows first the original data point, a perturbed version made by maximum damage adversarial attacks, and then the reconstruction given by the model. 
    As $\beta$ increases the models become more robust to attack. 
    }
    \label{fig:MNIST_recons}
    \vspace{-10pt}
\end{figure*}
\begin{figure*}[h]
         \centering
    \raisebox{0em}{
   \subfloat[Sampled $R^r_{\mathcal{X}}(\v{x})$
     ]{\includegraphics[height=0.215\textwidth]{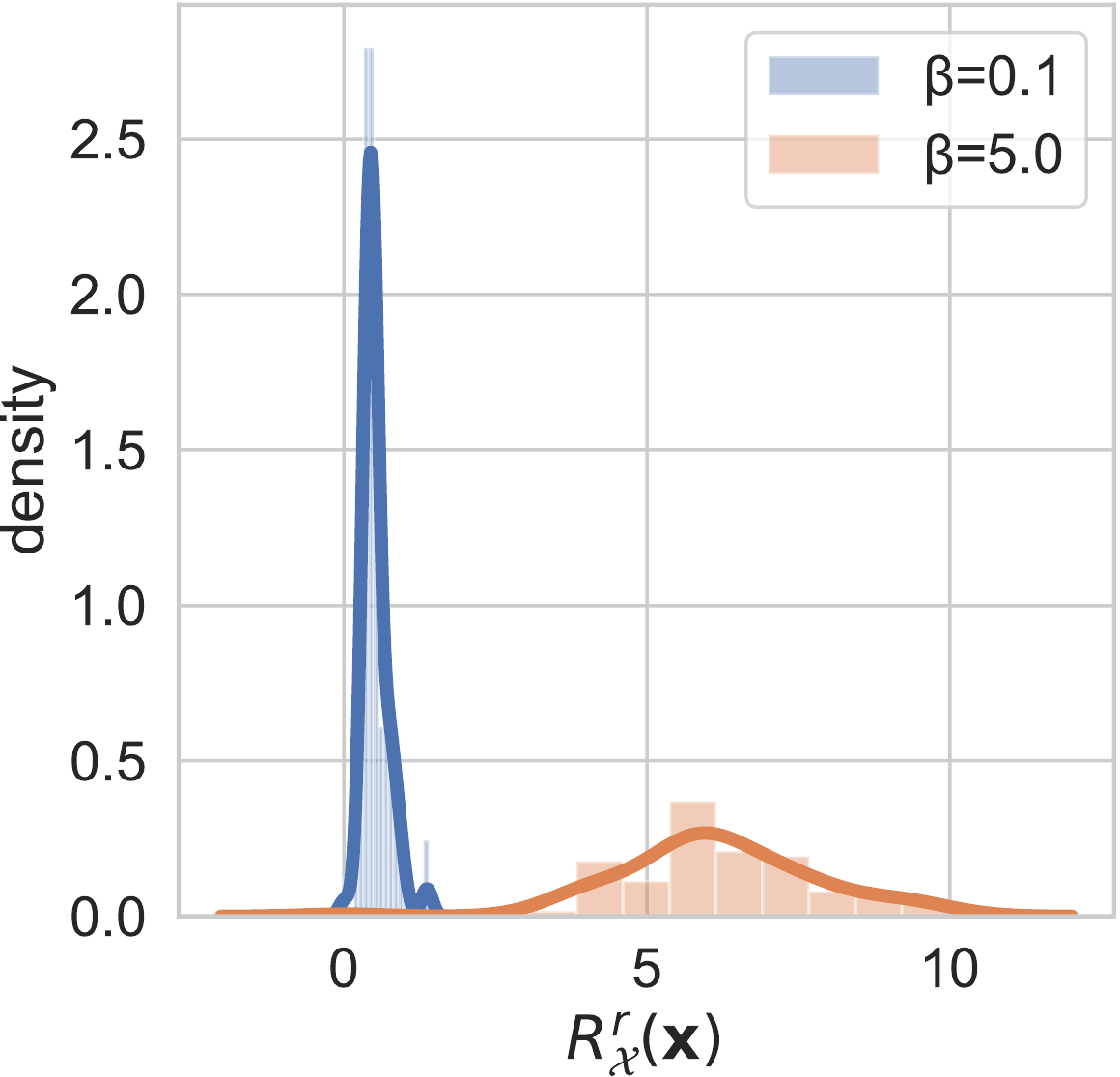}}}
    \subfloat[Adv Attack]{
    \includegraphics[width=0.24\textwidth]{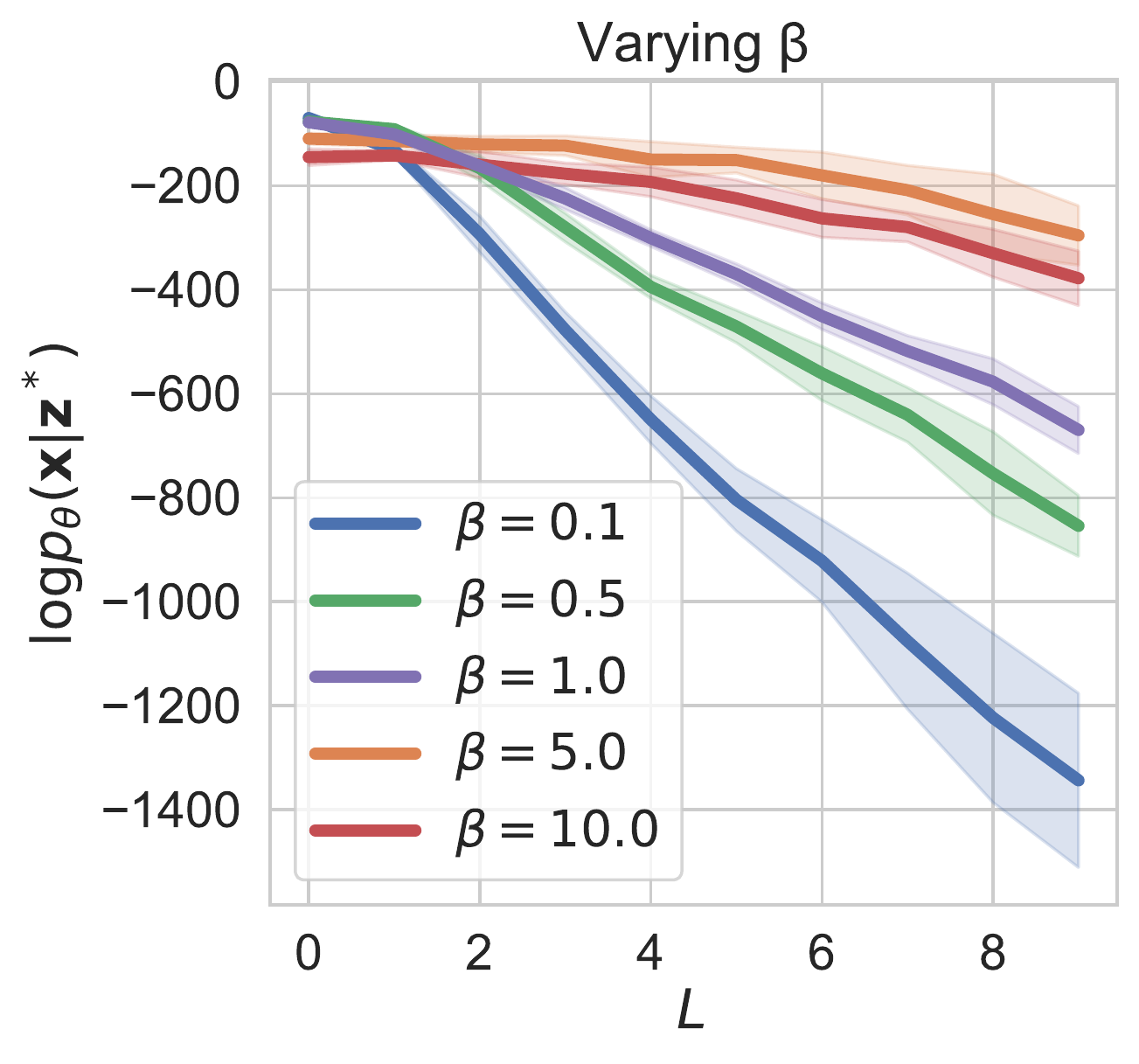}
    \label{fig:beta-adv}}
    \subfloat[][Encoder Variance]{\includegraphics[width=0.24\textwidth]{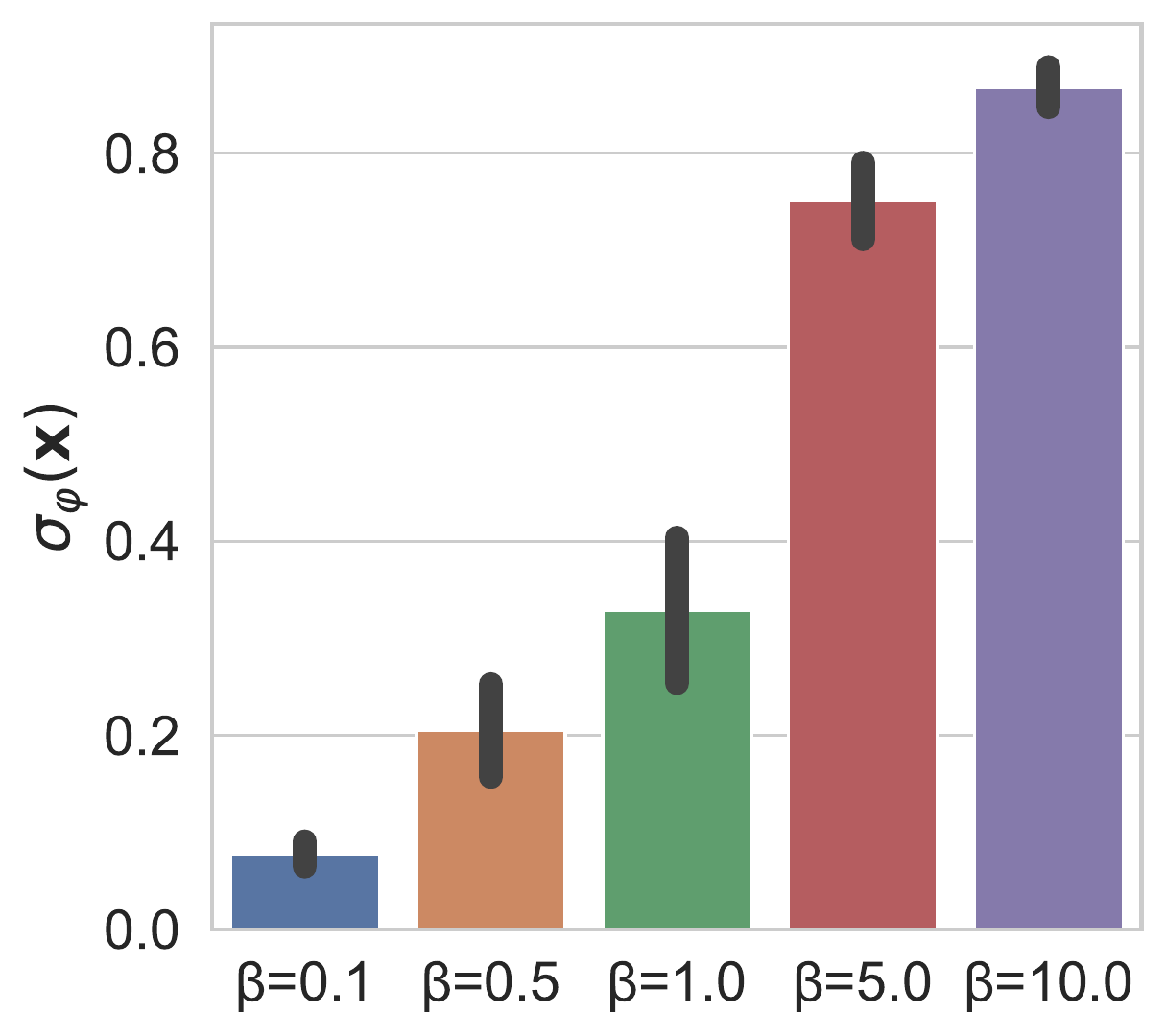}} 
     \subfloat[][Encoder Jacobian]{\includegraphics[width=0.24\textwidth]{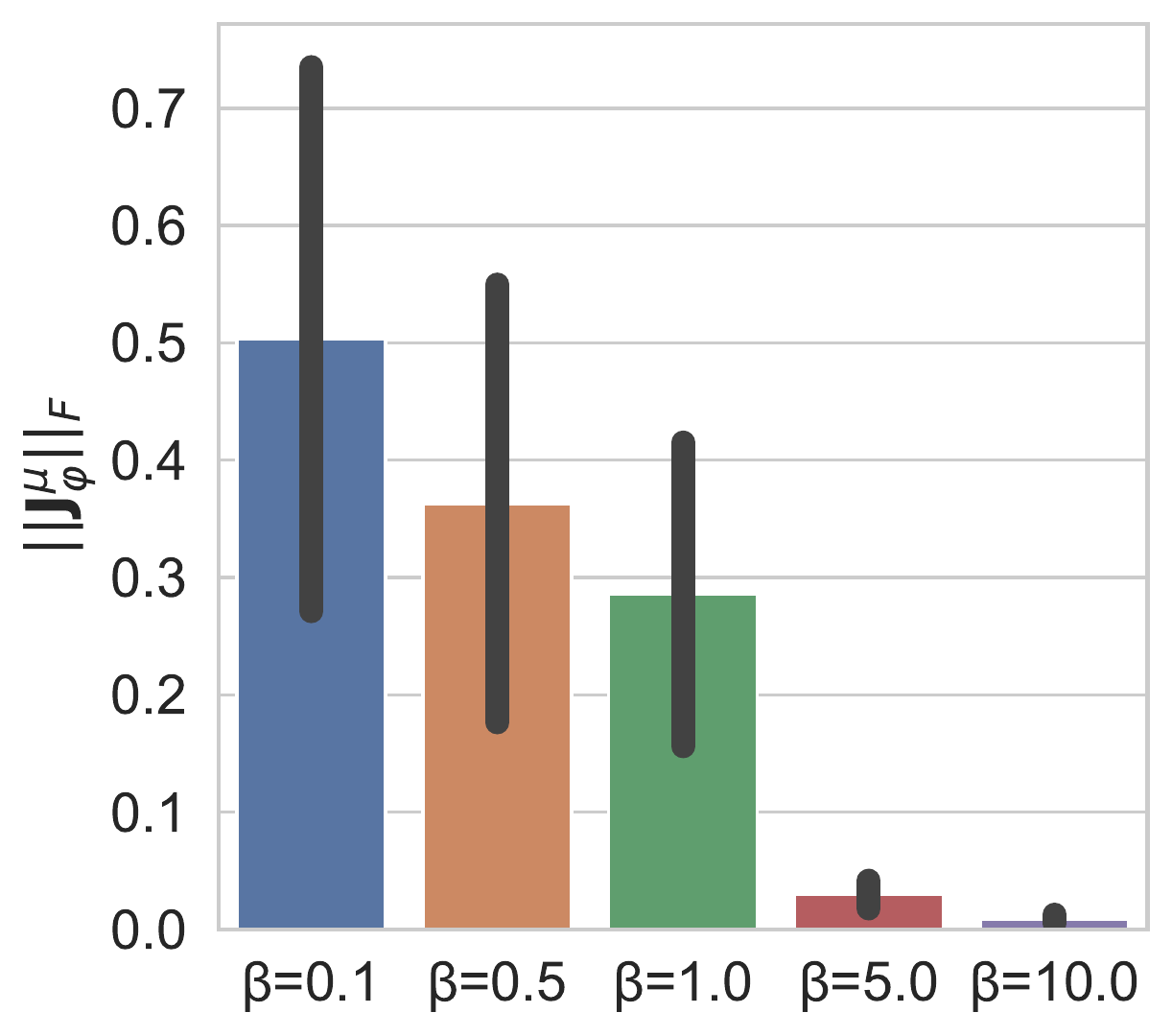}} 

    \caption{
     (a) distribution of the numerically estimated $R^r_{\mathcal{X}}(\v{x})$ ($m=0.5$) across the MNIST dataset.  We see that $R^r_{\mathcal{X}}(\v{x})$ increases dataset-wide for larger $\beta$.
    (b) likelihood of the original input given a maximum damage adversarial attack as in Eq \eqref{eq:adv_loss}.
    $L$ is the maximum allowed norm of the attack.  Large $\beta$ models retain high likelihoods even for large $L$, meaning they are robust to attack. 
    (c) and (d) show that the encoder variance increases and the encoder Jacobian norm ($||\v{J}^{\mu}_{\phi}(\v{x})||_F$) decreases as $\beta$ increases, supporting our analysis that the changes in these values underpin the robustness observed. 
    Confidence intervals for all plots are the standard deviation of values over the entire MNIST dataset.
    See Appendix \ref{app:bvae_exp} for similar experiments on other datasets.}
     \label{fig:mnist_experiments_beta}
\end{figure*}
We now apply our analysis to disentangling methods, which have empirically been shown to be more robust to adversarial attacks and noisy data~\citep{Willetts2019a}. 
First, we demonstrate this visually in Fig \ref{fig:MNIST_recons} where we see that $\beta$-VAEs are more resilient to attack as $\beta$ increases, and thus implicitly latent space overlap increases \citep{Mathieu2019}.
Second, we provide analysis to show that disentangling methods induce models with smaller encoder Jacobian norms \textit{and} larger 
posterior variances, implying that they have larger margins $R^r_\mathcal{X}(\v{x})$ by Theorem \ref{prop:margin_x}, a result we confirm empirically.

\paragraph{Disentangling increases encoder variance}
Empirically, increasing $\beta>1$ in a $\beta$-VAE increases the variance of the trained encoder, saturating at the variance of the prior for large $\beta$~\citep{Mathieu2019, Locatello2019}. 
We can shed light on this behavior by finding the optimum forms of the posterior distribution under these objective functions using calculus of variations.
We find that the optimal posterior has the form of a tempered or fractional posterior \citep{Holmes2017, Wenzel2020, miller2019robust}, with an exponent  $1/\beta$ on the likelihood:
\begin{restatable}{theorem}{BVAEopt}
For a $\beta$-VAE, the optimum posterior is:
\[q_\phi(\v{z}|\v{x}) \propto p(\v{z})p_\theta(\v{x}|\v{z})^{1/\beta}\]
\label{prop:beta_optim}
\vspace{-2em}
\end{restatable}
The proof is given in Appendix \ref{app:opt_post}.

This result gives the optimal posterior as a function of $\beta$.
It also tells us the $\beta$-VAE's optimal posterior in the limit of large $\beta$ is the prior, as we would expect.
Because the prior variance is naturally larger than that of the encoder, the encoder variance increases with $\beta$.
\paragraph{Disentangling penalizes Jacobian norm}
Assuming a Gaussian $p_\theta(\v{x}|\v{z})$, an encoder covariance optimal to first order, and activation functions that are piecewise-linear, the $\beta$--VAE objective can be approximated as~\citep{Kumar2020} 
\begin{align}
    \min_{\phi, \theta} \frac{1}{2}||\v{x} - g_{\theta}(\v{\mu}_\phi(\v{x}))||^2 + \frac{\beta}{2}||\v{J}^{\mu}_{\phi}(\v{x})\v{x}||^2_F \nonumber \\ + \frac{\beta}{2} \log |\v{I} + \frac{1}{\beta}\v{J}_{\theta}(\v{\mu}_\phi(\v{x}))\v{J}^T_{\theta}(\v{\mu}_\phi(\v{x}))|,
\label{eq:simple_VAE_obj}
\end{align}
As $\beta$ increases $||\v{J}^{\mu}_{\phi}(\v{x})||^2_F$ is more penalised and we expect to learn encoders with smaller Jacobians.

Taking these results together, we expect two things to occur as $\beta$ increases:
the encoder variance should increase by Theorem \ref{prop:beta_optim} and the norm of the encoder Jacobian should decrease.
We confirm this empirically in Fig \ref{fig:mnist_experiments_beta}(c,d). 
By Theorem \ref{prop:margin_x}
these two effects of increasing $\beta$ should increase $R^r_{\mathcal{X}}(\v{x})$ in tandem.
In Fig \ref{fig:mnist_experiments_beta}(a) we confirm that the numerical estimate for $R^r_{\mathcal{X}}(\v{x})$ increases \textit{dataset-wide} for large $\beta$, that is we get a larger value for $R_{\mathcal{X}}^r$, or metric for the overall robustness of the VAE.
In Appendix \ref{app:bvae_exp}, we also show that the distribution of the \emph{bound} for $R^r_{\mathcal{X}}(\v{x})$ from Theorem \ref{prop:margin_x} increases dataset-wide with increasing $\beta$. 
In both cases it is noticeable that $R^r_{\mathcal{X}}(\v x)$ is quite a well-behaved distribution, with reasonably low variance and skew.
This suggests that $R_{\mathcal{X}}^r$ can be reliably estimated in practice. 

In Fig~\ref{fig:beta-adv} we further show that these larger $R_{\mathcal{X}}^r$ values translate into larger likelihoods of the original input under adversarial attack, while 
in Appendix \ref{app:bvae_exp} we confirm that model sensitivity to noise is improved for larger $\beta$.
We note, however, that having a $\beta$ that is too large will completely undermine reconstructions~\citep{Higgins2017,Chen2018} and lead to VAEs that are never robust because we cannot confirm $r$-robustness, even without input perturbations (Eq~\eqref{eq:robust_satisfy}).
See Appendix \ref{app:emp_calc_pAR_res} for results on this.

\vspace{8pt}
\section{Conclusion}
We have defined a novel robustness metric tailored to probabilistic generative models, $r$-robustness, which can be used to assess the robustness of VAEs to adversarial attack. 
We defined a margin on a VAE's input space within which it is $r$-robust to perturbations and show that small norms of the encoder Jacobian and larger encoder variances are core contributors to robustness. 
Further, we offered theoretical and empirical analysis based on this margin, demonstrating that existing disentangling methods increase robustness by altering the optimal encoder variance and the norm of the encoder Jacobian.

\newpage
\acknowledgments{
This research was directly funded by the Alan Turing Institute under Engineering and Physical Sciences Research Council (EPSRC) grant EP/N510129/1.
AC was supported by an EPSRC Studentship.
MW was supported by EPSRC grant EP/G03706X/1.
SR gratefully acknowledges support from the UK Royal Academy of Engineering and the Oxford-Man Institute.
CH was supported by the Medical Research Council, the Engineering and Physical Sciences Research Council, Health Data Research UK, and the Li Ka Shing Foundation

We thank Tomas Lazauskas, Jim Madge and Oscar Giles from the Alan Turing Institute's Research Engineering team for their help and support.}

\bibliographystyle{apalike2}
\bibliography{references.bib}

\newpage
\appendix
\setcounter{equation}{0}
\setcounter{figure}{0}
\renewcommand\thefigure{\thesection.\arabic{figure}}
\setcounter{table}{0}
\renewcommand{\thetable}{\thesection.\arabic{table}}
\setcounter{definition}{0}
\renewcommand{\thedefinition}{\thesection.\arabic{definition}}
 \onecolumn
\clearpage
  \hsize\textwidth
  \linewidth\hsize \toptitlebar {\centering
  {\Large\bfseries Appendix for Towards a Theoretical Understanding of the Robustness of Variational Autoencoders \par}}
 \bottomtitlebar \vskip 0.1in
 \thispagestyle{appendixpage}

\section{Choosing $r$ for $r$-robustness}
\label{app:min_r}

\begin{prop}
For any input and perturbation, a necessary requirement for a VAE with a Gaussian encoder to satisfy $r$-robustness is that
 \begin{equation}
     r > \sqrt{2\mathrm{Tr}(\v{\Sigma}(\v{x})})+\mathcal{O} (\v \varepsilon)
 \end{equation}
 where $\v{\Sigma}(\v{x}) = \v{J}_\theta( \v{\mu}_\phi(\v{x}))\v{\sigma}^2_\phi(\v{x})\v{J}^T_\theta( \v{\mu}_\phi(\v{x}))$,  $(\v{J}_\theta(\v{\mu}_\phi(\v{x}))_{i,j} = \partial g_{\theta}(\v{\mu}_\phi(\v{x}))_{i}/\partial (\v{\mu}_\phi(\v{x}))_j$, and $\mathcal{O} (\v \varepsilon)$ represents higher order terms that tend to zero in the limit $\v \sigma_{\phi}(\v x)\to \v 0$.
\label{prop:r-bound}
 \end{prop}
 
We provide empirical confirmations in Appendix \ref{app:emp_calc_pAR_res} that show that the $r$ for $r$-robustness scales with the encoder variance. 

\begin{proof} Let $g_{\theta}(\v{\mu}_\phi(\v{x}))$ be the result of mapping to the encoder mean ($\v{\mu}_\phi$) and then decoding to the likelihood mean ($g_{\theta}$), and $\Delta(\v{x}) = g_{\theta}(\v{\mu}_\phi(\v{x}) + \v{\eta}\circ \v{\sigma}_\phi(\v{x})) - g_{\theta}(\v{\mu}_\phi(\v{x}))$ we want to find a bound for $r$ for which: 

\begin{align}
    p(||\Delta(\v{x})||_2 \leq r) &> p(||\Delta(\v{x})||_2 > r)
\end{align}

where as before $\v{\eta}\sim \mathcal{N}(\v{0}, \v{I})$. 

Here we can invoke Taylor's theorem on $g_{\theta}(\v{\mu}_\phi(\v{x}) + \v{\eta}\circ \v{\sigma}_\phi(\v{x}))$ around the deterministic mapping $ \v{\mu}_\phi(\v{x})$.
Namely, if we assume that all terms in Hessian of $g_{\theta}(\v{\mu}_\phi(\v{x})$ are finite (i.e.~$|\partial^2  g_{\theta}(\v{\mu}_\phi(\v{x})_i / \partial  \v{\mu}_\phi(\v{x})_j \v{\mu}_\phi(\v{x})_k|<\infty\, \forall i,j,k$), then we have:
\begin{align}
\label{eqapp:del_z}
    g_\theta(\v{\mu}_\phi(\v{x}) + \v{\epsilon}) = g_\theta(\v{\mu}_\phi(\v{x})) + \v{J}_{\theta}(\v{\mu}_\phi(\v{x}))(\v{\eta}\circ \v{\sigma}_\phi(\v{x}))  +\mathcal{O}(\v \varepsilon)
\end{align}

where $\mathcal{O}(\varepsilon)$ represents asymptotically dominated higher order terms that go to zero in the limit of small $\v{\sigma}_\phi(\v{x}))$ and  
$\v{J}_\theta$ is 
 defined element-wise as: 
 \begin{equation}
     \v{J}_\theta(\v{\mu}_\phi(\v{x}))_{i,j} = \frac{\partial g_{\theta}(\v{\mu}_\phi(\v{x}))_{i}}{\partial (\v{\mu}_\phi(\v{x}))_j}
 \end{equation}
Note that $\v{J}_{\theta}(\v{\mu}_\phi(\v{x}))(\v{\eta}\circ \v{\sigma}_\phi(\v{x}))$ is distributed according to the multivariate Gaussian \[\mathcal{N}(0, \v{J}_{\theta}(\v{\mu}_\phi(\v{x}))\v{\sigma}^2_\phi(\v{x})\v{J}^T_{\theta}(\v{\mu}_\phi(\v{x})))\] 
Given these definitions 
\begin{align}
    &p(||\Delta(\v{x})||_2 \leq r) > p(||\Delta(\v{x})||_2 > r) \\
    &\Leftrightarrow p(||\Delta(\v{x})||_2 \leq r)  > 0.5 \\
    &\Leftrightarrow p(|| \v{J}_{\theta}(\v{\mu}_\phi(\v{x}))(\v{\eta}\circ \v{\sigma}_\phi(\v{x}))  +\mathcal{O}(\v \varepsilon)||_2 < r)>0.5   \\
    &\Leftrightarrow p(||\v{J}_{\theta}(\v{\mu}_\phi(\v{x}))(\v{\eta}\circ \v{\sigma}_\phi(\v{x}))  +\mathcal{O}(\v \varepsilon)||^2_2 < r^2)>0.5 
\end{align}

We must now consider the distribution of the square norm of $\v{\mathcal{E}}(\v{x}) = \v{J}_{\theta}(\v{\mu}_\phi(\v{x}))(\v{\eta}\circ \v{\sigma}_\phi(\v{x}))$. 
Let 
\begin{align}
    Q(\v{\mathcal{E}}(\v{x})) &= ||\v{\mathcal{E}}(\v{x})||^2_2 = \v{\mathcal{E}}(\v{x})^T \v{\mathcal{E}}(\v{x}) \\
    \v{\Sigma}(\v{x}) &= \v{J}_{\theta}(\v{\mu}_\phi(\v{x}))\v{\sigma}^2_\phi(\v{x})\v{J}^T_{\theta}(\v{\mu}_\phi(\v{x})) \\
    \v{Y}(\v{x}) &= \v{\Sigma}(\v{x})^{-\frac{1}{2}}\v{\mathcal{E}}(\v{x})
\end{align}

Given that we restrict ourselves to positive activation functions, $\v{J}_{\theta}$ is positive and $\v{\Sigma}(\v{x})$ will be positive semi definite and is invertible.
As such we have $Q(\v{\mathcal{E}}) = \v{Y}^T(\v{x})\v{\Sigma}(\v{x})  \v{Y}(\v{x})$. 

Using the spectral decomposition theorem we can write that $\v{\Sigma}(\v{x})  = \v{P}^T(\v{x}) \v{\Lambda}(\v{x}) \v{P}(\v{x})$ where $\v{P}^T(\v{x})\v{P}(\v{x}) = \v{I}$ and $\v{\Lambda}(\v{x})$ is the diagonal matrix of the eigenvalues of $\v{\Sigma}(\v{x})$,  $\lambda_1, \dots , \lambda_{d_\mathcal{X}}$, where $d_\mathcal{X}$ is the dimensionality of the data-space. 
Given that $\v{\Sigma}(\v{x})$ is positive semi definite $\v{\Lambda}(\v{x})$ will only have positive values. 

Let $\v{U}(\v{x}) = \v{P}(\v{x})\v{Y}(\v{x}) =   \v{P}(\v{x})\v{\Sigma}(\v{x}) ^{-\frac{1}{2}}\v{\mathcal{E}}(\v{x})$, which is multivariate Gaussian with identity matrix and zero mean. 
We have that: 

\begin{align}
    Q(\v{\mathcal{E}}) &=  \v{Y}^T(\v{x})\v{\Sigma}(\v{x})\v{Y}(\v{x}) \\
    &= \v{Y}^T(\v{x}) \v{P}^T(\v{x}) \v{\Lambda}(\v{x}) \v{P}(\v{x})\v{Y}(\v{x}) \\
    &= \v{U}^T(\v{x}) \v{\Lambda}(\v{x}) \v{U}(\v{x})
\end{align}

As such: 

\begin{equation}
    \sum_{i=1}^{d_{\mathcal{X}}} (\v{\mathcal{E}}_i)^2  = \v{U}^T(\v{x}) \v{\Lambda}(\v{x}) \v{U}(\v{x}) =  \sum_{i=1}^{d_{\mathcal{X}}} \\ \lambda_i(\v{U}_i(\v{x}))^2, \quad \lambda_i(\v{U}_i(\v{x}))^2 \sim  \Gamma \left(\frac{1}{2}, 2\lambda_i   \right)
\end{equation}

This comes from the fact that for $\lambda_i\v{X}, \v{X} \sim \Gamma \left(\frac{1}{2}, 2 \right)$ we have that $ \lambda_i\v{X} \sim \Gamma \left(\frac{1}{2}, 2\lambda_i\right)$.

To establish a lower bound on $r$, we use Markov's inequality which states that: 
\begin{equation}
    p(||\v{\mathcal{E}}(\v{x})  +\mathcal{O}(\v \varepsilon)||^2_2 > r^2) < \frac{\expect||\v{\mathcal{E}}(\v{x})  +\mathcal{O}(\v \varepsilon)||^2_2}{r^2}
\end{equation}
Here $\expect||\v{\mathcal{E}}(\v{x}) ||^2_2 = \expect \sum_{i=1}^{d_{\mathcal{X}}} (\v{\mathcal{E}}_i(\v{x}))^2 = \expect \sum_{i=1}^{d_{\mathcal{X}}} (\lambda_i(\v{U}_i)^2(\v{x}))^2$, which is simply $\sum_{i=1}^{d_{\mathcal{X}}} \lambda_i$. 

Recall that we want: $p(||\v{\mathcal{E}}(\v{x})  +\mathcal{O}(\v \varepsilon)||^2_2 > r^2) < 0.5$. As such 
 \begin{equation}
  r > \sqrt{2\sum_{i=1}^{d_{\mathcal{X}}} \lambda_i} +\mathcal{O}(\v \varepsilon) =\sqrt{2\mathrm{Tr}(\v{\Sigma(\v{x})})} +\mathcal{O}(\v \varepsilon)
  \end{equation}
\end{proof}

\section{Margin for $r$-robustness in $\mathcal{X}$}
\label{app:margin_X}

\Rbound*
\begin{proof}
Suppose we have an $r$ for which $r$-robustness is satisfied before any perturbation is added to the VAE input. 
First we want to establish a margin in the latent space $\mathcal{Z}$ for which our model is robust given a perturbation in the latent space. 

To do this, we first define
\begin{align}
   \Delta_{e}(\v{y}) = g_{\theta}(\v{y}) - g_{\theta}(\v{\mu}_\phi(\v{x})), \quad \v y \in \mathcal{Z}
\end{align}
where $g_{\theta}$ is the decoder network and $\v y$ is an arbitrary realization of the latents.
Note here that there is an implicit dependency on $\v x$, but as this input is fixed we will ignore this dependency throughout.
Let $\v{z} = \v{\mu}_\phi(\v{x}) + \v{\eta} \circ \v{\sigma}_\phi(\v{x})$ be the random variable produced by the embedding, i.e.~the latent sampled by the encoder.
We want to find a bound $R^r_e$ for which: 
\begin{align}
\label{eq:Rz}
\lVert \v{\delta}_z \rVert_2 \le R^r_e \quad \Leftrightarrow \quad
p(||\Delta_{e}(\v{z}+\v{\delta}_z)||_2 \leq r) &> p(||\Delta_{e}(\v{z}+\v{\delta}_z)||_2 > r)
\end{align}
such that $r$-robustness is satisfied on the decoder output when we apply a deterministic a perturbation $\v{\delta}_z$ of maximum size $R_e^r$ to the random variable $\v z$.
Note that all the stochasticity is contained in $\v \eta$.

Let $A^r$ denote the set of $\v{\delta}_z$ for which~\eqref{eq:Rz} holds
and conversely let $B^r$ be the set of $\v{\delta}_z$ for which it does not. %
By assumption in the Theorem, then $\v 0 \in A^r$ as the unperturbed input satisfies $r$-robustness.
Moreover, we also have that this unperturbed input $\v \delta_z$ has a probability $p_{\Delta}(\v 0) := p(||\Delta_{e}(\v{z})||_2 \leq  r)=p(||\Delta(\v{x})||_2 \leq  r)>0.5$ of returning a reconstruction with $r$ of $g_{\theta}(\v{\mu}_\phi(\v{x}))$.

Now we know that $\v z$ is a Gaussian random variable and so regardless of form of the decoder, $p_{\Delta}(\v \delta_z) := p(||\Delta_{e}(\v{z}+\v \delta_z)||_2\le r)$ must vary smoothly as we change $\v \delta_z$.
In essence, as we increase the size of the perturbation $\v \delta_z$ slowly from zero, the distribution of $\v z+\v \delta_z$ will still have most of its mass of the same region $\v z$.
When coupled with the fact that we have some ``excess probability'' $p_{\Delta}(\v 0)-0.5$ beyond what it is needed for $r$-robustness, there must be at certain degree to which we can increase $\v \delta_z$ before all this excess probability is ``used up''.
We can then use this to construct a bound for $R_e^r$ by considering the minimum $\v \delta_z$ to break $r$-robustness in the ``worst-case'' setting for the boundary between $A^r$ and $B^r$.

Intuitively as shown in Figure~\ref{fig:boundary_in_z}, and also more formally using the Neyman-Pearson lemma \citep{Neyman1992} by analogy to the approach of~\cite{Cohen2019}, this worst case setting will occur when the boundary between $A^r$ and $B^r$ is a straight line perpendicular to the direction of lowest variance for $\v z$ (remembering that this is Gaussian distributed) and $\v \delta_z$ is increased in this direction of lowest variance.
In essence, this is the setup where our excess probability is used up most quickly for a given $\lVert \v \delta_z\rVert_2$.
By assumption in the theorem statement, we are using a diagonal covariance encoder and so this direction of lowest variance is the latent variable corresponding to $\argmin_i \v \sigma_{\phi}(\v x)_i$.
Further, by noting that we need only consider the marginal distribution in this dimension, it is straightforward to see that the bound is reached when
\begin{align}
\label{eq:delta_z_bound}
    \lVert \v \delta_z\rVert_2 = \left(\min_i \v \sigma_{\phi}(\v x)_i\right) \Phi^{-1} \left(p_{\Delta}(\v 0)\right)
    = \left(\min_i \v \sigma_{\phi}(\v x)_i\right) \Phi^{-1} \left(p(||\Delta(\v{x})||_2 \leq  r)\right)
\end{align}
where $\Phi^{-1}$ is the inverse cumulative distribution function for a unit Gaussian, i.e.~the probit function.
Note that this yields $\lVert \v \delta_z\rVert_2=0$ if $p(||\Delta(\v{x})||_2 \leq  r)=0$, such we get the excepted result that our margin is zero is $r$-robustness only just holds without an input perturbation.

Next we need to relate $\lVert \v \delta_z\rVert_2$ to $\lVert \v \delta_x\rVert_2$.
Here we can straightforwardly invoke Taylor's theorem on $\v \mu_{\phi}(\v x+\v \delta_x)$ around the original input $\v x$.
Namely, if we assume that all terms in Hessian of $\mu_{\phi}(\v x)$ are finite (i.e.~$|\partial^2 \v \mu_{\phi}(\v x)_i / \partial \v x_j \v x_k|<\infty\, \forall i,j,k$), then we have
\begin{align}
\label{eq:del_z}
    \v \delta_z = \v \mu_{\phi}(\v x+\v \delta_x)-\v \mu_{\phi}(\v x) = \v{J}^{\mu}_{\phi}(\v{x})\v{\delta}_x +\mathcal{O}(\varepsilon)
\end{align}
where $\mathcal{O}(\varepsilon)$ represents asymptotically dominated higher order terms that go to zero in the limit of small $\v \delta_x$.
We thus have
\begin{align}
    \lVert \v \delta_x\rVert_2  \leq \frac{\lVert \v \delta_z\rVert_2}{\lVert \v{J}^{\mu}_{\phi}(\v{x})\rVert_{F}}+\mathcal{O}(\varepsilon)
\end{align}
where $\mathcal{O}(\varepsilon)$ again represents asymptotically dominated higher order terms (note though these are not the same terms as in~\eqref{eq:del_z}).
To complete the proof we now simply combine this with~\eqref{eq:delta_z_bound} to give the $\lVert \v \delta_x\rVert_2$ at which the bound is reached and thus the $R_{\mathcal{X}}^r(\v x)$ quoted in the theorem, namely
\begin{align}
R^r_\mathcal{X}(\v{x}) \geq \frac{(\min_i \v{\sigma}_\phi(\v{x})_i)\Phi^{-1}(p(||\Delta(\v{x})||_2 \leq r))}{||\v{J}^{\mu}_{\phi}(\v{x})||_F } + \mathcal{O} (\varepsilon)
\end{align}
where the inequality comes from the fact that the $\v \delta_z$ we derived was the worst possible case (i.e. smallest $\v \delta_z$ which might reach the bound).

\begin{figure}[t]
    \centering
    \includegraphics[width=0.35\textwidth]{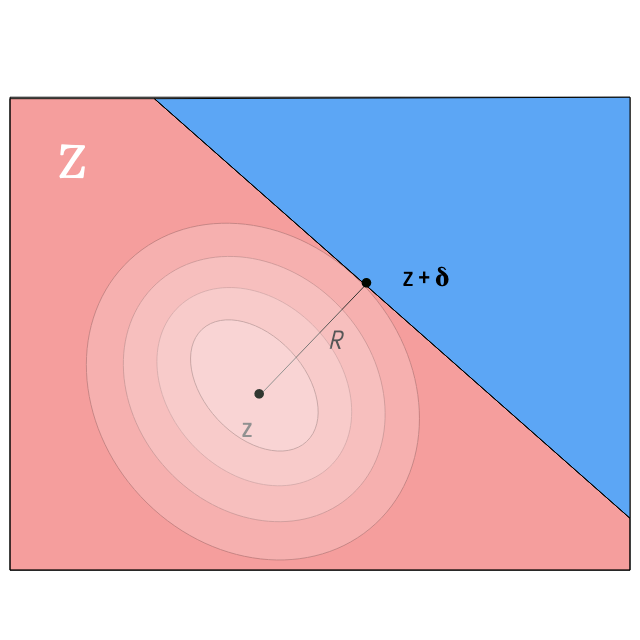} 
    \caption{Illustration of the boundary $R$ we are measuring in $\mathcal{Z}$. Red represents spaces where $A^r$ is satisfied. Blue represent spaces where $B^r$ is satisfied. 
    The concentric ellipsoids centered on $\v{z}$ are the contours of $\mathcal{N}(\v{\mu}_\phi(\v{x}), \v{\sigma}^2_\phi(\v{x}))$. $R$ is the minimum distance $\delta$ for which $A^r$ is satisfied.
    The line dividing the two spaces represent the Neyman-Pearson ``worst-case'' model and is along the direction of minimum \textit{variance}, $\min_i \v{\sigma}^2_\phi(\v{x})_i$. 
    }
    \label{fig:boundary_in_z}
\end{figure}

\end{proof}

\newpage
\section{$\beta$-VAE Optimal Posterior}
\label{app:opt_post}

\BVAEopt*

\begin{proof}
Here we use calculus of variations to obtain optimal posteriors for $\beta$-VAEs.
The objectives we are optimising are over the whole dataset $\mathcal{D}=\{\v{x}_i\}, i=1,..,N$, with empirical data density $\rho(\v{x})=\frac{1}{N}\sum_i^N \delta(\v{x}-\v{x}_i)$.

The evidence lower bound for a $\beta$-VAE is
\begin{equation}
    \ELBO_\beta(\mathcal{D};\theta,\phi)=\expect_{\rho(\v{x})}\big[\expect_{q_\phi(\v{z}|\v{x})}[\log p_\theta(\v{x}|\v{z})] - \beta \KL(q_\phi(\v{z}|\v{x})||p(\v{z}))\big].
\end{equation}

This is easier to work with written explicitly as integrals.
Note that as we are going to be finding the optimal $q_\phi(\v{z}|\v{x})$ we must add a constraint so that it integrates to 1.
\begin{equation}
    \ELBO_\beta(\mathcal{D};\theta,\phi)=\int \intdx \intdz \rho(\v{x})  \big[ q_\phi(\v{z}|\v{x}) [ \log p_\theta(\v{x}|\v{z}) -\beta \log q_\phi(\v{z}|\v{x}) + \beta \log p(\v{z})] + \lambda(\v{x})(q_\phi(\v{z}|\v{x}) -1)\big]
    \label{eq:bvae-elbo-int}
\end{equation}

For brevity, going forward $p_\theta(\v{x}|\v{z}) = p$, $p(\v{z}) = \pi$, $q_\phi(\v{z}|\v{x}) = q$.
We also view $\ELBO$ as depending on $q, p$ directly.

To proceed with calculus of variations, we substitute $q\rightarrow q+\epsilon$, where $\epsilon$ is a small function that goes to zero appropriately fast for large $\v{x}, \v{z}$.
Thus we expand $\ELBO$ to first order in $q$ to find $\frac{\delta \ELBO}{\delta q}$.
The form of $q$ for which this gradient is zero gives us the optimum $q$ for this functional.

\begin{equation}
    \ELBO_\beta(q+\epsilon)=\int \intdx \intdz \rho(\v{x}) \big[ (q+\epsilon)[ \log p -\beta \log(q + \epsilon) + \beta \log \pi] + \lambda(\v{x})(q + \epsilon -1)\big]
\end{equation}

Recall that $\log (1+x) \approx x$ to first order.
Thus $\log (q + \epsilon) \approx \log q + \frac{\epsilon}{q}$ to first order.
So,
\begin{align}
    \ELBO_\beta(q+\epsilon)=&\int \intdx \intdz \rho(\v{x}) q \big[ \log p -\beta \log q + \beta \log \pi - \beta \frac{\epsilon}{q} \big]\\
    &+\int \intdx \intdz \rho(\v{x}) \epsilon \big[ \log p -\beta \log q + \beta \log \pi - \beta \frac{\epsilon}{q} \big]\\
    &+ \int \intdx \intdz \rho(\v{x}) \lambda(\v{x})(q -1)
    + \int \intdx \intdz \rho(\v{x}) \lambda(\v{x})\epsilon
    +  \int \intdx \intdz \rho(\v{x}) O(\epsilon^2).
\end{align}
Rearranging we find
\begin{align}
    \ELBO_\beta(q+\epsilon)=& \ELBO_\beta(q) +\int \intdx \intdz \rho(\v{x}) \epsilon \big[ \log p -\beta \log q + \beta \log \pi - \beta + \lambda(\v{x}) \big]
    \nonumber \\
    &\qquad\qquad\qquad +  \int \intdx \intdz \rho(\v{x}) O(\epsilon^2) \\
    =& \ELBO_\beta(q) + \int \intdx \intdz \frac{\delta\ELBO_\beta}{\delta q} \epsilon + \int \intdx \intdz \rho(\v{x}) O(\epsilon^2)
\end{align}
At the optimum value of $q$ the functional will have vanishing functional derivative $\frac{\delta\ELBO_\beta}{\delta q}$, so
\begin{align}
    & \log p -\beta \log q + \beta \log \pi - \beta + \lambda(\v{x}) = 0,\\
    & \log q = \frac{1}{\beta} \log p + \log \pi + C(\v{x}).
    \end{align}
Exponentiating we find the optimal $q$ to be
\begin{equation}
q_\phi(\v{z}|\v{x}) = \frac{1}{Z(\v{x})} p_\theta(\v{x}|\v{z})^{\frac{1}{\beta}}p(\v{z}),
\end{equation}
where $Z$ is an appropriate normalising constant.
This completes the proof.
\end{proof}

\section{Empirical Calculation of the Bounds}
\label{app:emp_calc}

\subsection{Estimating the minimum $r$}
\label{app:emp_calc_pAR}

\subsubsection{Results}
\label{app:emp_calc_pAR_res}

\begin{figure}[h]
    \centering
    \subfloat[][]{\includegraphics[width=0.3\textwidth]{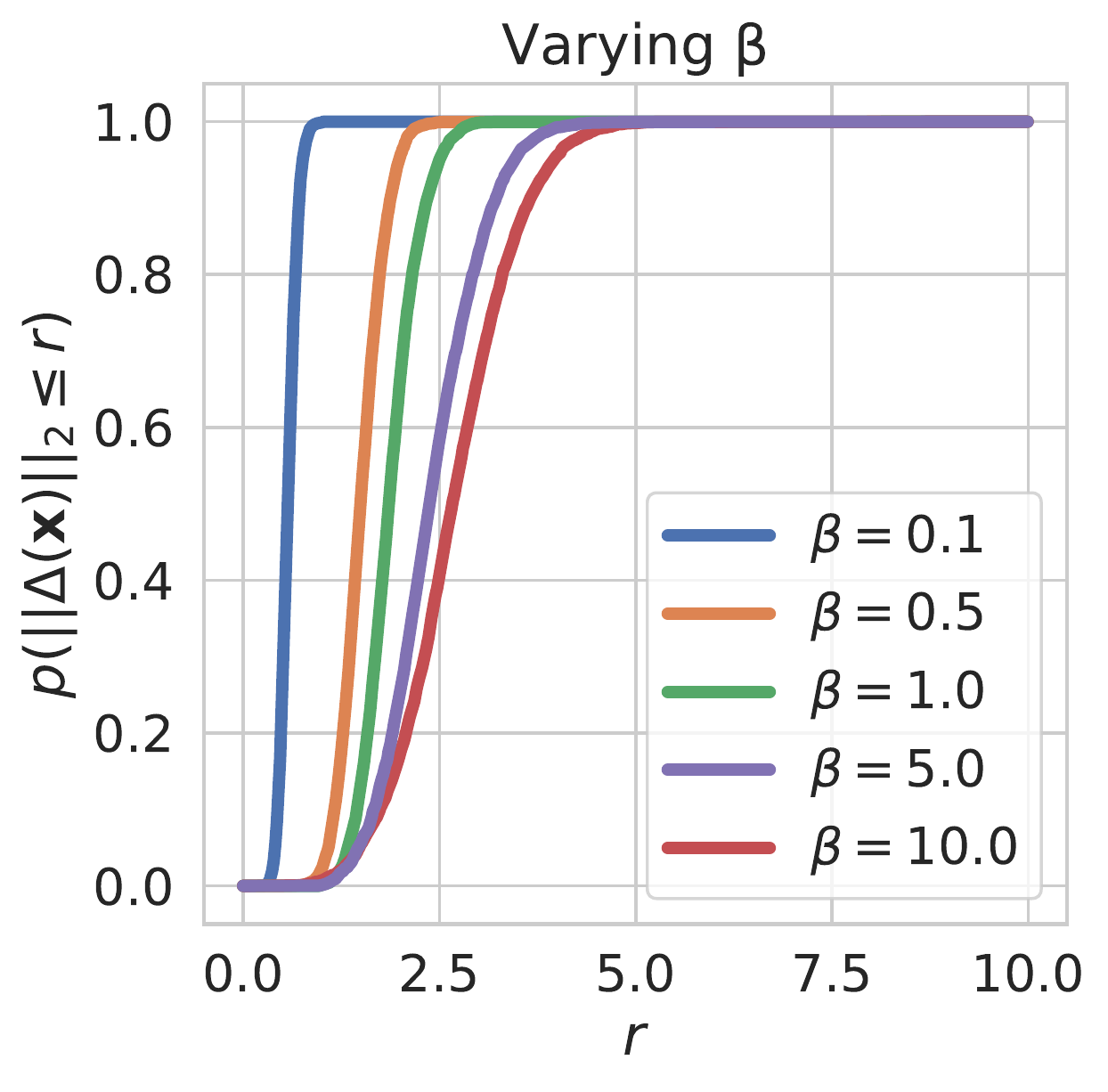}}
    \subfloat[][]{\includegraphics[width=0.3\textwidth]{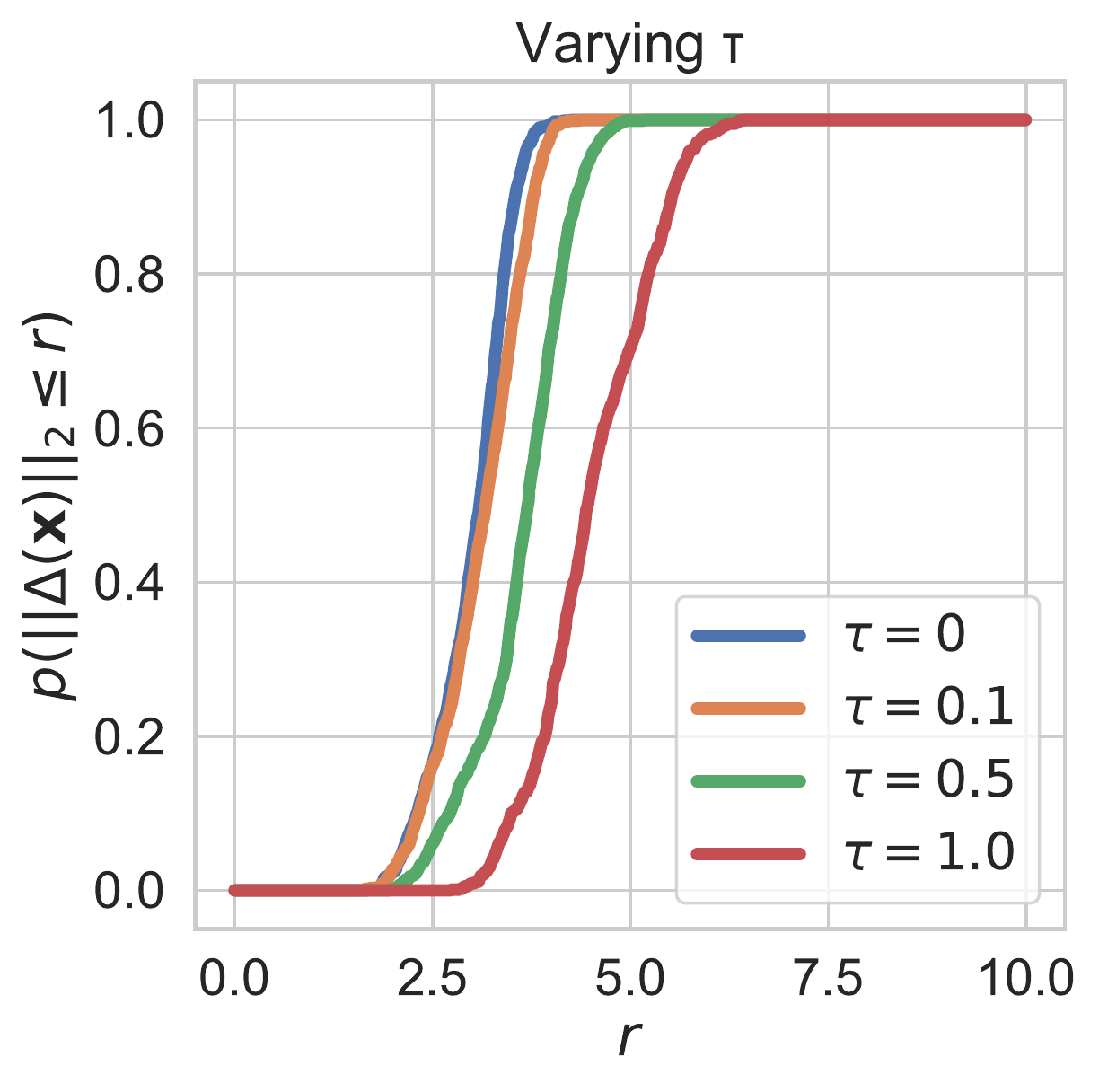}}
    \caption{
    Here we show that the minimum $r$ for which $p(||\Delta(\v{x})||_2 \leq r) = 0.5$ increases with $\beta$ and $\tau$, where $\beta$ is the penalty applied to the $\KL$ in $\beta$-VAEs and $\tau$ is an offset added to the encoder standard deviation $\v{\sigma}_{\phi}(\v{x})$.
    This probability, estimated as detailed below in Appendix \ref{app:emp_calc_algo}, increases with $r$, but 
    increases more slowly for large $\beta$ (a) and large $\tau$ (b).
    In such models the encoding process has higher variance resulting in a greater spread of reconstructions, confirming Proposition \ref{prop:r-bound} in Appendix A that the minimum $r$ for $r$-robustness increases with the encoder variance.
    }
    \label{fig:PAR_estimates}
\end{figure}

\subsubsection{Algorithm}
\label{app:emp_calc_algo}

\begin{algorithm}[H]
 \caption{Estimating $r$}
\SetAlgoLined
\KwResult{$r$ such that  $p(||\Delta(\v{x})||_2 \leq r)>0.5$}
 $m$, $step$, $samples$, $\v{x}$,  $r \gets 0$, $p(||\Delta(\v{x})||_2 \leq r) \gets 0$\;
 \While{$p(||\Delta(\v{x})||_2 \leq r) < m$}{
 $d \gets \{\}$\;
 \For{$i\gets1$ \KwTo $samples$ \KwBy $1$}{
  $\v{s} \sim \mathcal{N}(\v{\mu}_{\phi}(\v{x}), \v{\sigma}_{\phi}(\v{x}))$\;
 $s_d\gets ||g_{\theta}(\v{s}) - g_{\theta}(\v{\mu}_\phi(\v{x}))||_2$\;
  $d$.insert($s_d$)\; 
  }
  $r \gets r + step$ \;
  $p(||\Delta(\v{x})||_2 \leq r) \gets \frac{\mathrm{Sum}(d < r)}{nsamples}$\;
 }
\end{algorithm}

\subsection{Estimating $R^r_{\mathcal{X}}(\v{x})$}
\label{app:emp_calc_R}

\begin{algorithm}[H]
 \caption{Estimating $R^r_{\mathcal{X}}(\v{x})$}
\SetAlgoLined
\KwResult{$R^r_{\mathcal{X}}(\v{x})$ such that  $p(||\Delta(\v{x}, \v{\delta}_x)||_2 \leq r)>0.5$}
 $step$, $samples$, $\v{x}$,  $r$, $p(||\Delta(\v{x}, \v{\delta}_x)||_2 \leq r) \gets 0$, $R^r_{\mathcal{X}}(\v{x})\gets 10$, $restarts\gets 5$ \;
 \While{$p(||\Delta(\v{x}, \v{\delta}_x)||_2 \leq r) < 0.5$}{
 \For{$j\gets1$ \KwTo $restarts$ \KwBy $1$}{
 $d \gets \{\}$\;
 \For{$i\gets1$ \KwTo $samples$ \KwBy $1$}{
  $\v{\delta}_x \gets$ max damage attack constrained to the norm $R^r_{\mathcal{X}}(\v{x})$\;
  $\v{s} \sim \mathcal{N}(\v{\mu}_\phi(\v{x} + \v{\delta}_x), \v{\sigma}_\phi(\v{x} + \v{\delta}_x))$ \\
 $s_d\gets ||g_{\theta}(\v{s}) - g_{\theta}(\v{\mu}_\phi(\v{x}))||_2$\;
  $d$.insert($s_d$)\; 
  }
  $R^r_{\mathcal{X}}(\v{x}) \gets R^r_{\mathcal{X}}(\v{x}) - step$\;
  $p(||\Delta(\v{x}, \v{\delta}_x)||_2 \leq r) \gets \frac{\mathrm{Sum}(d < r)}{nsamples}$\;}
  }
\end{algorithm}

\newpage
\section{$\beta$-VAE Sensitivity Experiments}
\label{app:bvae_exp}

\begin{figure}[h!]
         \centering
    \subfloat[][$\beta=0.1$]{\includegraphics[width=0.2\textwidth]{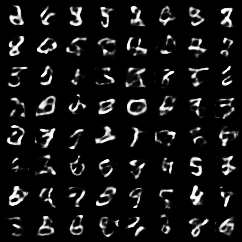}} 
    \subfloat[][$\beta=0.5$]{\includegraphics[width=0.2\textwidth]{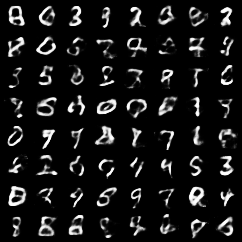}} 
    \subfloat[][$\beta=1$]{\includegraphics[width=0.2\textwidth]{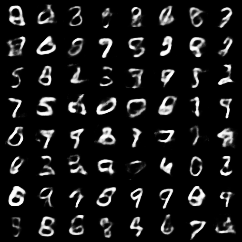}} 
    \subfloat[][$\beta=5$]{\includegraphics[width=0.2\textwidth]{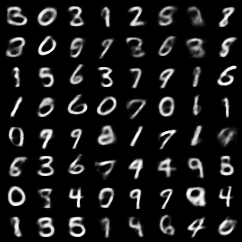}} 
     \subfloat[][$\beta=10$]{\includegraphics[width=0.2\textwidth]{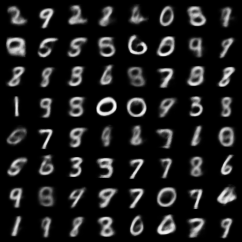}} \\
     \subfloat[][VAE Sensitivity]{\includegraphics[width=0.25\textwidth]{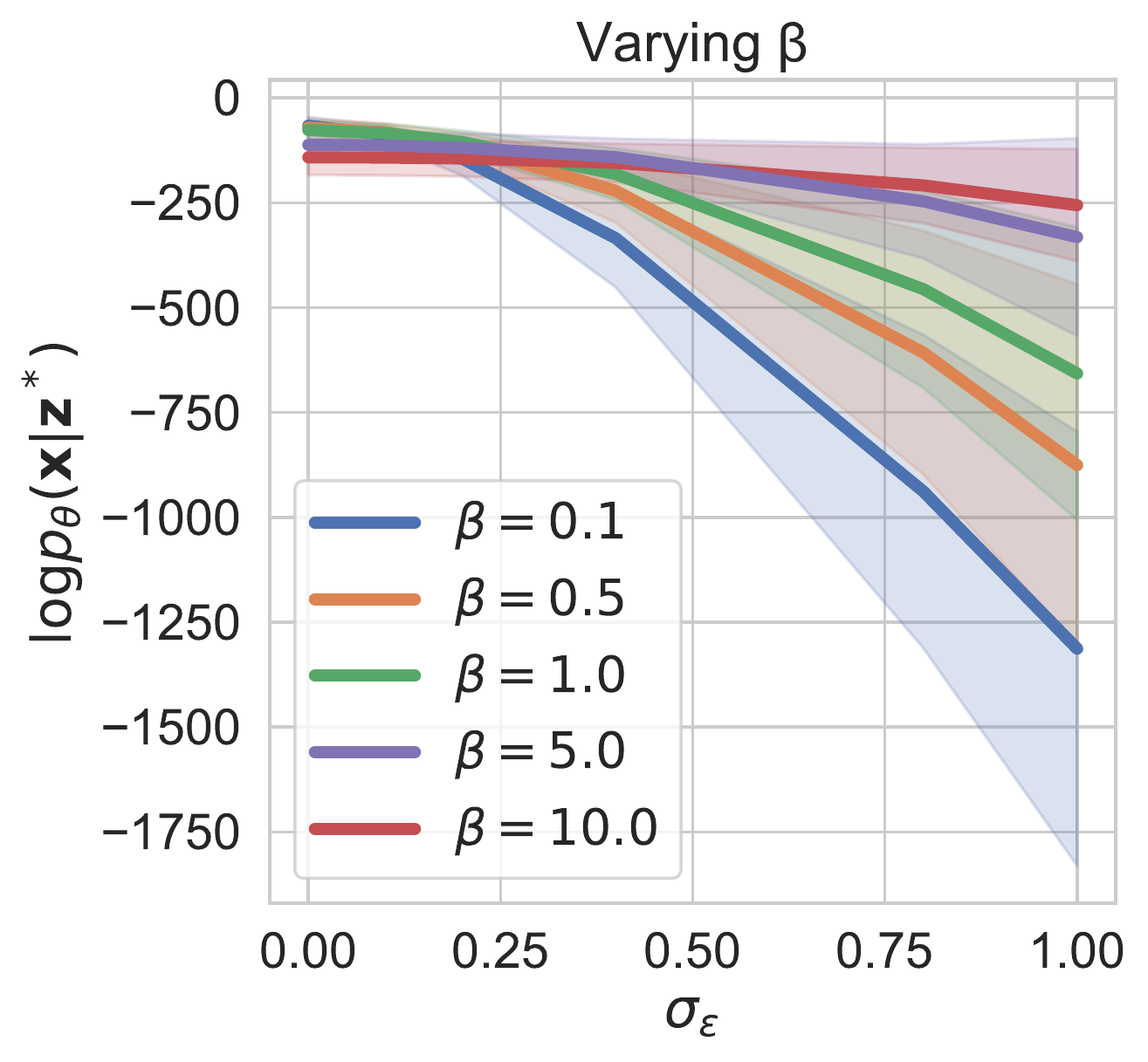}} 
    \subfloat[][Encoder Variance]{\includegraphics[width=0.25\textwidth]{figures/mnist/mnist_beta_sig.pdf}} 
     \subfloat[][Encoder Jacobian]{\includegraphics[width=0.25\textwidth]{figures/mnist/mnist_beta_norm_enc.pdf}} 
     \subfloat[][$R^r_{\mathcal{X}}(\v{x})$ bound]{\includegraphics[width=0.25\textwidth]{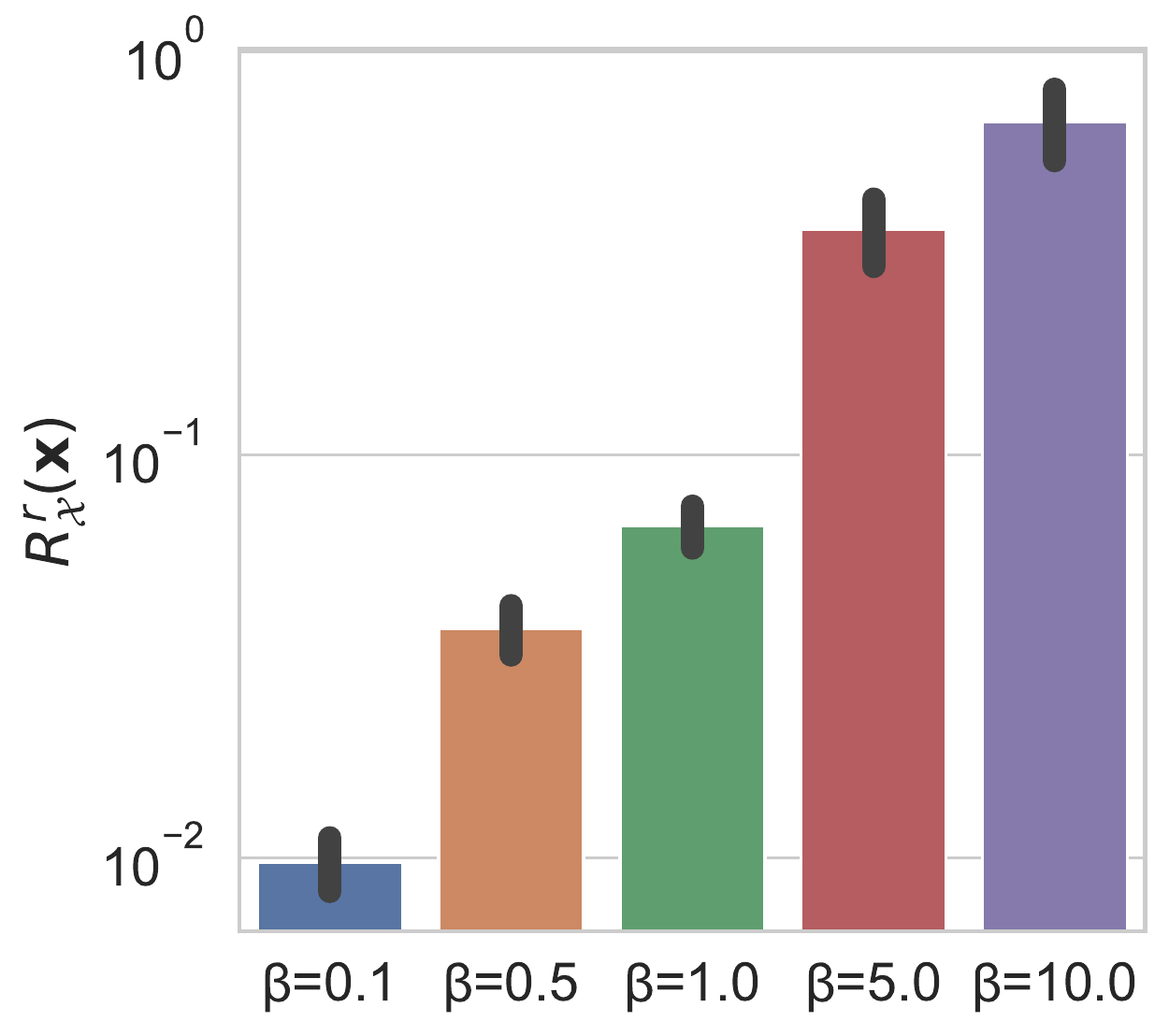}} \\
    
    \caption{Here we illustrate that $\beta$-VAEs, trained MNIST, with higher $\beta$ penalties generalise better and are less sensitive to input perturbations. 
    The first row (a)-(e) shows samples drawn from the latent space prior that are then fed through the VAE decoder. 
    It is clear that as $\beta$ increases, so too does the quality of generated samples.
    (f) shows the sensitivity of the VAE to input perturbations. We add zero-mean Gaussian noise of variance $\sigma^2_{\epsilon}$ to the VAE input to form a noisy input $\v{x}^*$ and embedding $\v{z}*$. We then measure the likelihood of the original point $\v{x}$ under this noisy embedding. 
    $\sigma^2_{\epsilon}$ is thus an approximation of the margin of robustness of the VAE, if the VAE's likelihood does not change even for high variance noise, it must have a large margin of robustness ($R^r_\mathcal{X}(\v{x})$). 
    The likelihood of $\v{x}$ is quasi constant, under increasing noise variance, for high values of $\beta$. This supports our analysis that such models have higher $R^r_\mathcal{X}(\v{x})$. 
    Figures (g) and (h) show that the encoder variance and that the norm of the encoder Jacobian ($||\v{J}^{\mu}_{\phi}(\v{x})||_F$) increase as $\beta$ increases, supporting our analysis that the changes in these values underpin the robustness observed.
     In (i) we calculate the bound for $R^r_{\mathcal{X}}(\v{x})$ from Theorem \ref{prop:margin_x} where we ignore higher order terms.
    We select $r$ such that $p_{A^r}(\v{x})=0.9$, which is a relatively strict metric for robustness.
    In (f-i) confidence intervals correspond to the standard deviations of values over the entire MNIST dataset.
    Taken as a whole these experiments support our analysis that the margin $R^r_{\mathcal{X}}(\v{x})$ increases with $\beta$ as in Theorem \ref{prop:beta_optim}, in conjuction with the norm of the encoder Jacobian and the encoder variance, supporting Theorem \ref{prop:margin_x}.}
     \label{fig:mnist_experiments_beta_app}
\end{figure}

\begin{figure}[h!]
         \centering
    \subfloat[][$\beta=0.1$]{\includegraphics[width=0.2\textwidth]{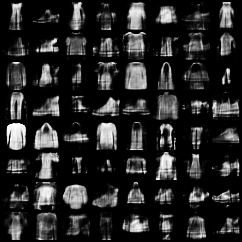}} 
    \subfloat[][$\beta=0.5$]{\includegraphics[width=0.2\textwidth]{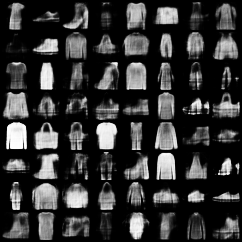}} 
    \subfloat[][$\beta=1$]{\includegraphics[width=0.2\textwidth]{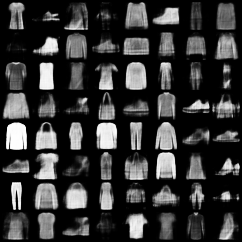}} 
    \subfloat[][$\beta=5$]{\includegraphics[width=0.2\textwidth]{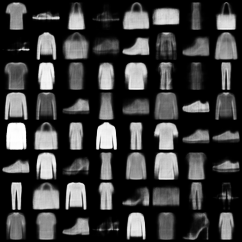}} 
     \subfloat[][$\beta=10$]{\includegraphics[width=0.2\textwidth]{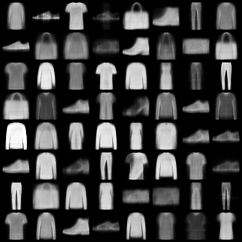}} \\
     \subfloat[][VAE Sensitivity]{\includegraphics[width=0.25\textwidth]{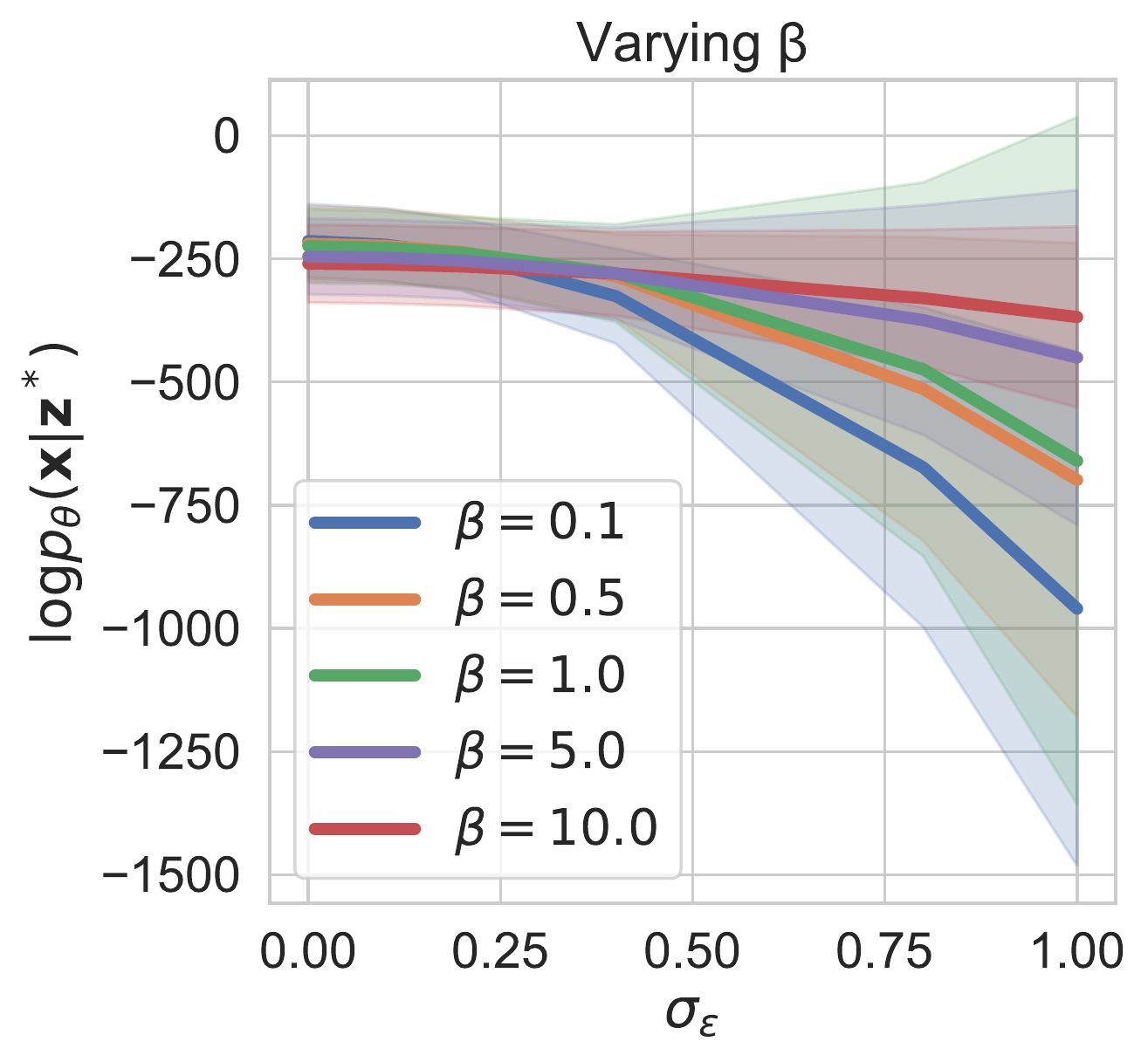}} 
    \subfloat[][Encoder Variance]{\includegraphics[width=0.25\textwidth]{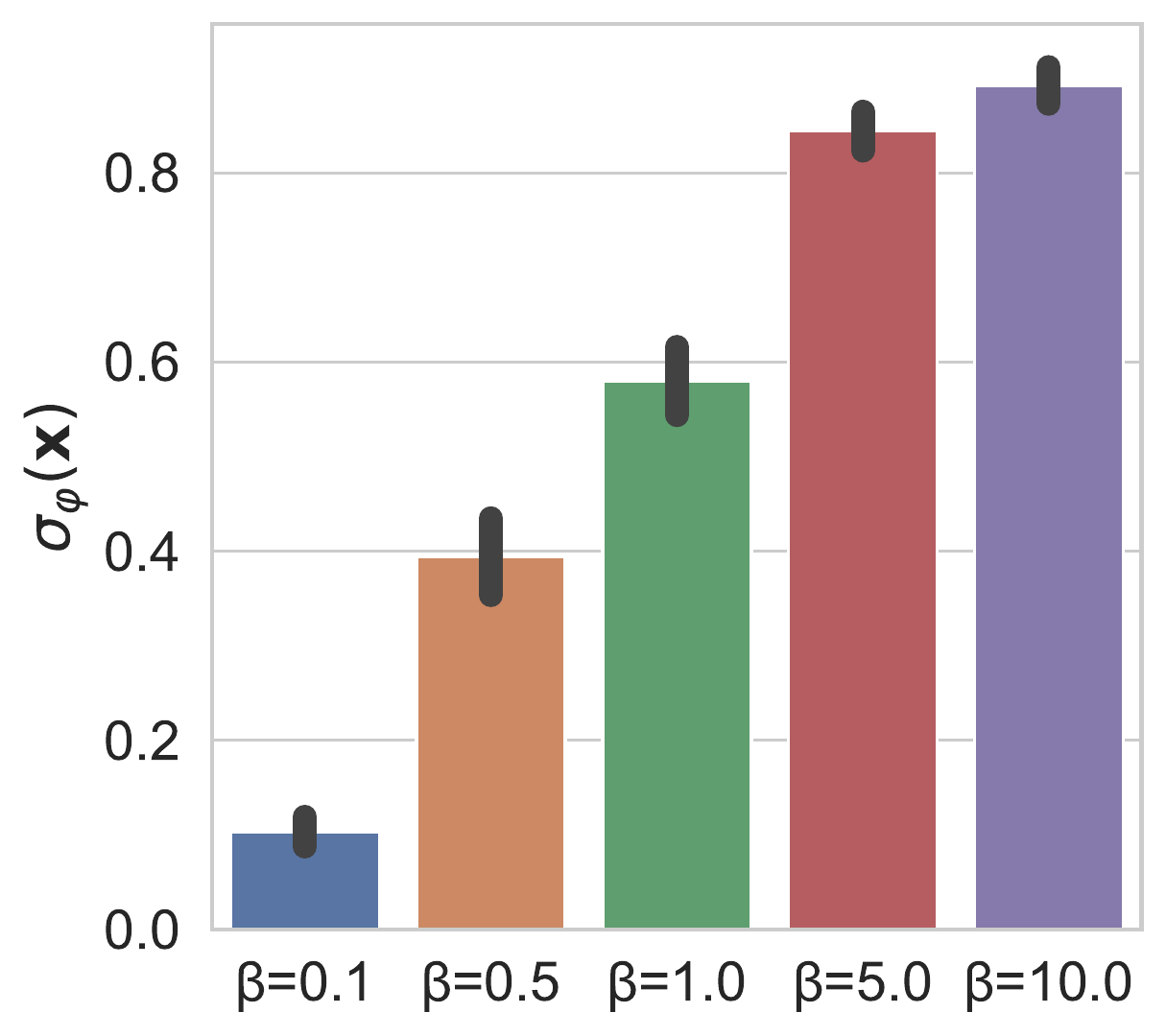}} 
     \subfloat[][Encoder Jacobian]{\includegraphics[width=0.25\textwidth]{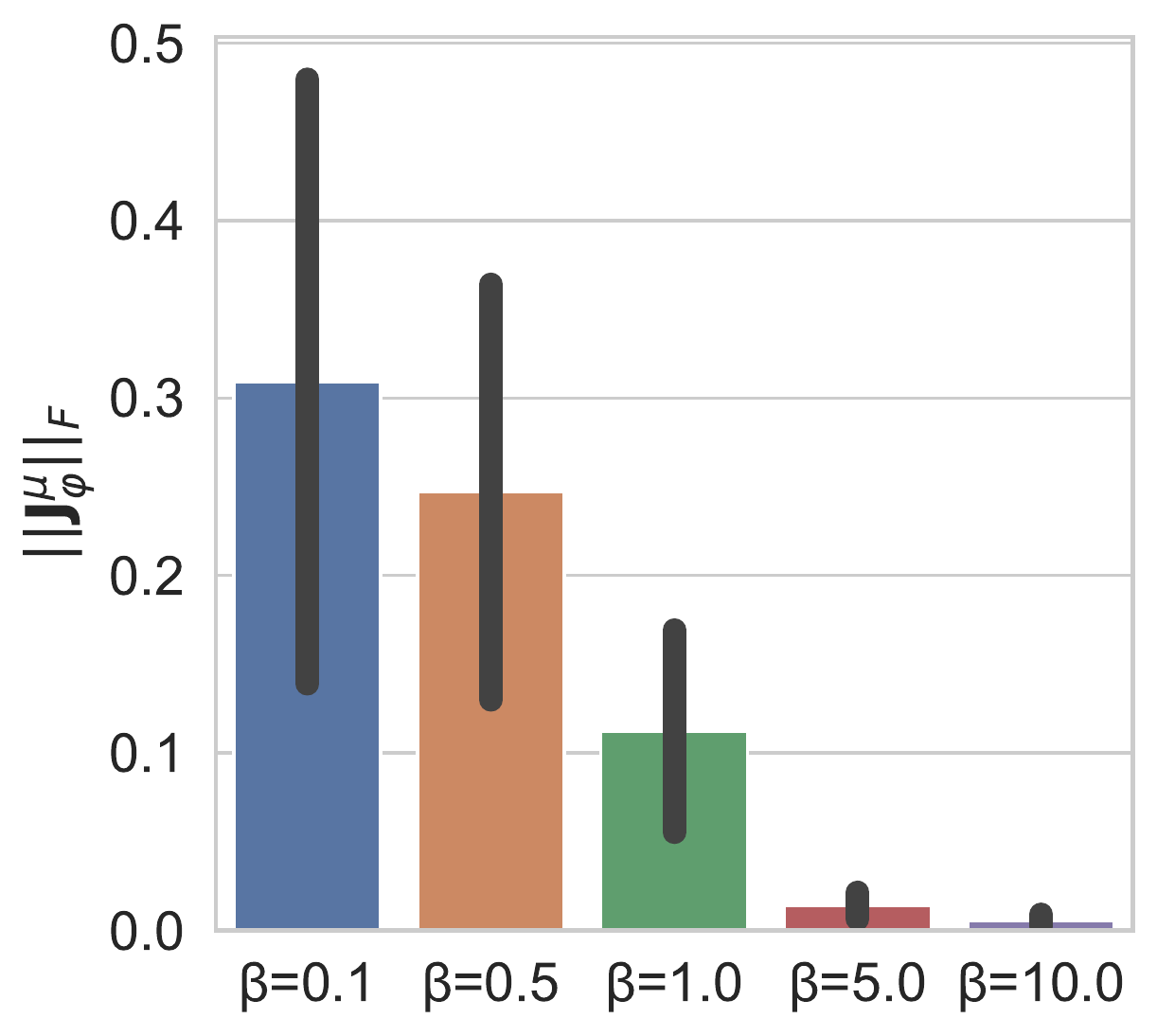}} 
     \subfloat[][$R^r_{\mathcal{X}}(\v{x})$ bound]{\includegraphics[width=0.25\textwidth]{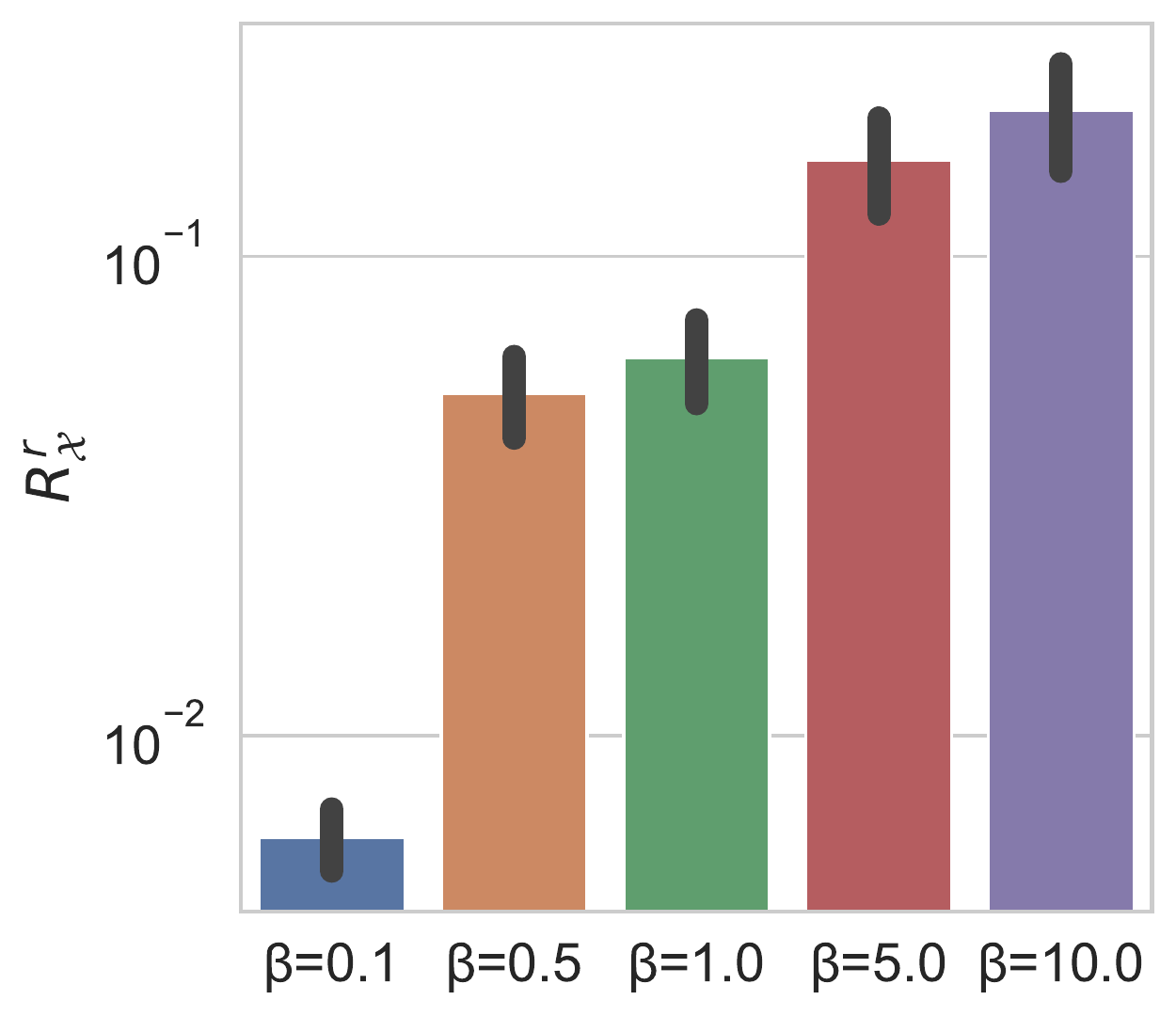}}\\
    
    \caption{Here we illustrate that $\beta$-VAEs, trained on fashion-MNIST, with higher $\beta$ penalties generalise better and are less sensitive to input perturbations. 
    The first row (a)-(e) shows samples drawn from the latent space prior that are then fed through the VAE decoder. 
    It is clear that as $\beta$ increases, so too does the quality of generated samples. 
    (f) shows the sensitivity of the VAE to input perturbations. We add zero-mean Gaussian noise of variance $\sigma^2_{\epsilon}$ to the VAE input to form a noisy input $\v{x}^*$ and embedding $\v{z}*$. We then measure the likelihood of the original point $\v{x}$ under this noisy embedding. 
    $\sigma^2_{\epsilon}$ is thus an approximation of the margin of robustness of the VAE, if the VAE's likelihood does not change even for high variance noise, it must have a large margin of robustness ($R^r_\mathcal{X}(\v{x})$). 
    The likelihood of $\v{x}$ is quasi constant, under increasing noise variance, for high values of $\beta$. This supports our analysis that such models have higher $R^r_\mathcal{X}(\v{x})$. 
    Figures (g) and (h) show that the encoder variance and that the encoder Jacobian norm ($||\v{J}^{\mu}_{\phi}(\v{x})||_F$) increase as $\beta$ increases, supporting our analysis that the changes in these values underpin the robustness observed.
     In (i) we calculate the bound for $R^r_{\mathcal{X}}(\v{x})$ from Theorem \ref{prop:margin_x} where we ignore higher order terms.
    We select $r$ such that $p_{A^r}(\v{x})=0.9$, which is a relatively strict metric for robustness.
    In (f-i) confidence intervals correspond to the standard deviations of values over the entire fashion-MNIST dataset.
    Taken as a whole these experiments support our analysis that the margin $R^r_{\mathcal{X}}(\v{x})$ increases with $\beta$ as in Theorem \ref{prop:beta_optim}, in conjuction with the norm of the encoder Jacobian and the encoder variance, supporting Theorem \ref{prop:margin_x}.}
     \label{fig:fmnist_experiments_beta_app}
\end{figure}

\begin{figure}[h!]
         \centering
    \subfloat[][Encoder Variance]{\includegraphics[width=0.25\textwidth]{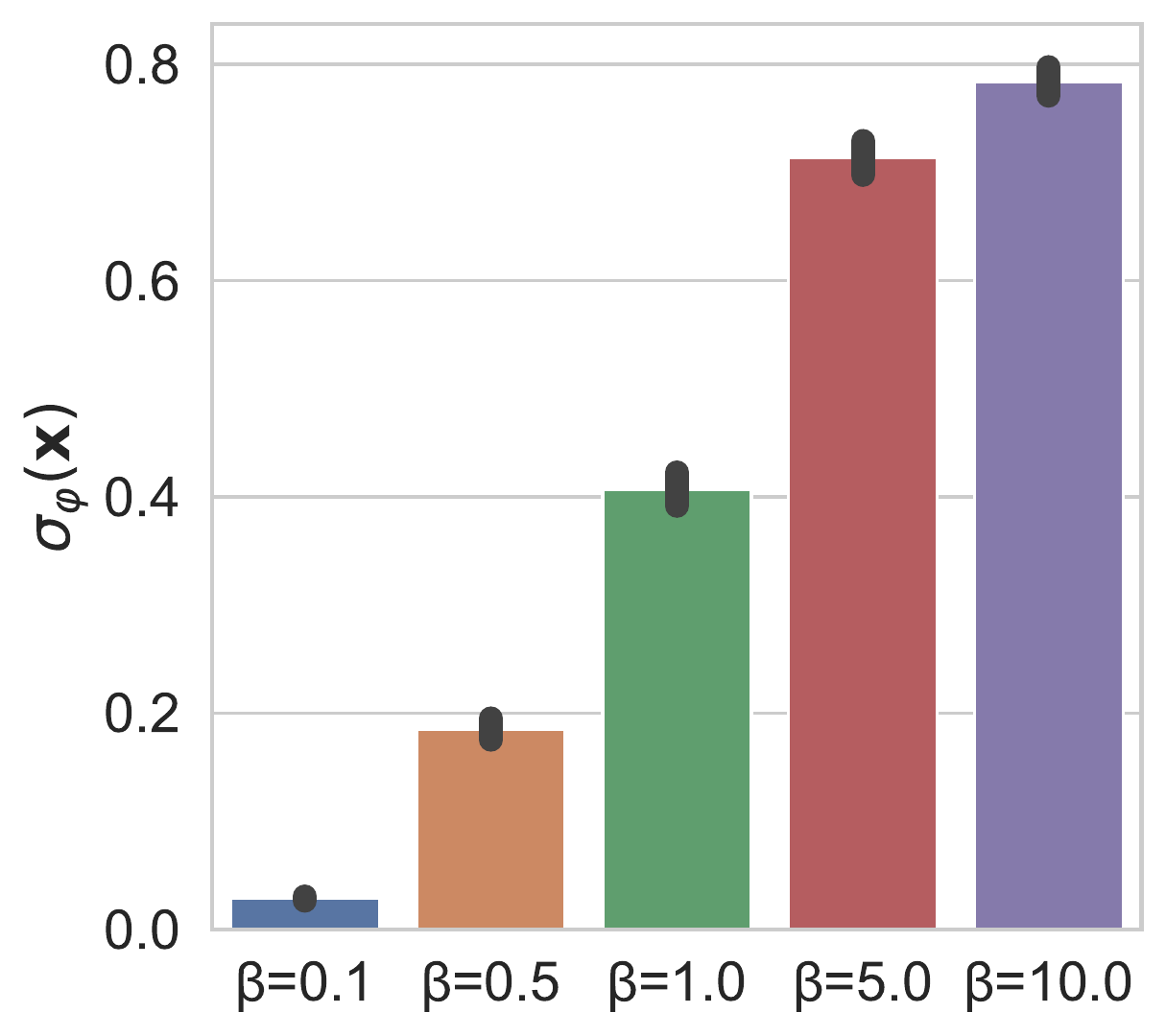}} 
     \subfloat[][Encoder Jacobian]{\includegraphics[width=0.25\textwidth]{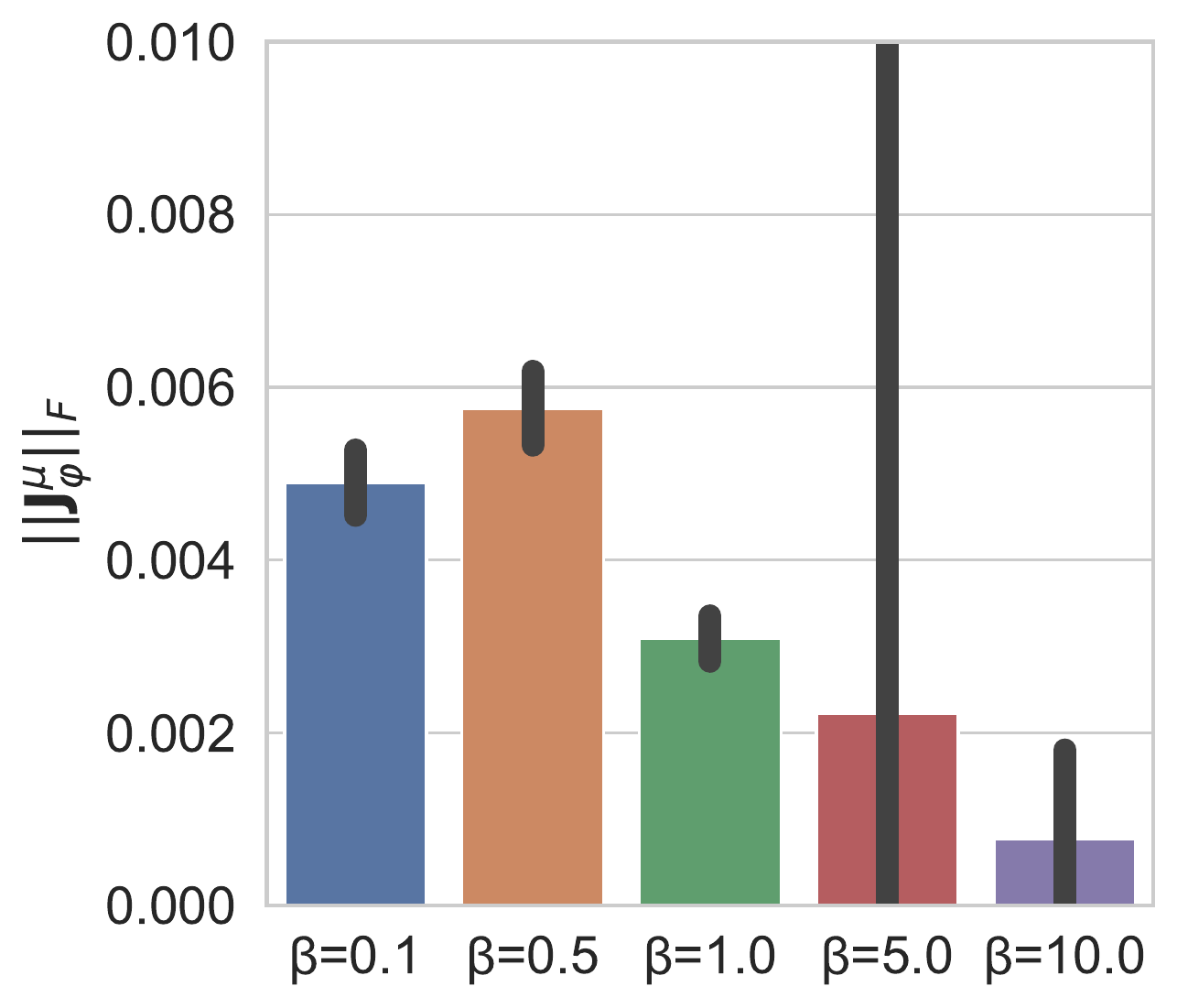}} 
     \subfloat[][$R^r_{\mathcal{X}}(\v{x})$ bound]{\includegraphics[width=0.25\textwidth]{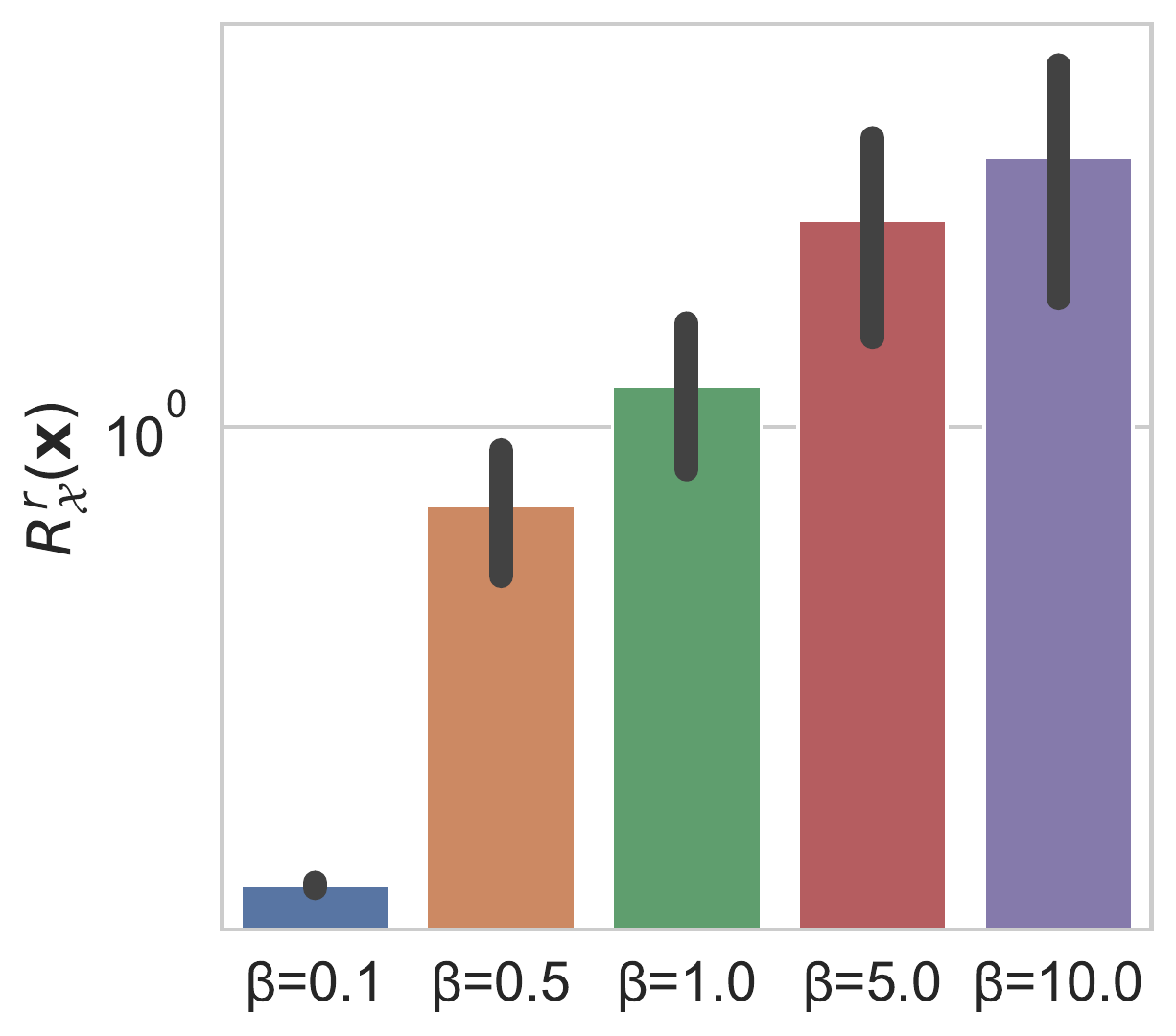}}\\
    
    \caption{Here we illustrate that $\beta$-VAEs, trained on CIFAR10, with higher $\beta$ have larger margins of robustness. 
    Figures (a) and (b) show that the encoder variance and that the encoder Jacobian norm ($||\v{J}^{\mu}_{\phi}(\v{x})||_F$) increase as $\beta$ increases, supporting our analysis that the changes in these values underpin the robustness observed.
     In (i) we calculate the bound for $R^r_{\mathcal{X}}(\v{x})$ from Theorem \ref{prop:margin_x} where we ignore higher order terms.
    We select $r$ such that $p_{A^r}(\v{x})=0.9$, which is a relatively strict metric for robustness.
    In (a-c) confidence intervals correspond to the standard deviations of values over the entire dataset.
    Taken as a whole these experiments support our analysis that the margin $R^r_{\mathcal{X}}(\v{x})$ increases with $\beta$ as in Theorem \ref{prop:beta_optim}, in conjuction with the norm of the encoder Jacobian and the encoder variance, supporting Theorem \ref{prop:margin_x}.}
     \label{fig:cifar10_experiments_beta_app}
\end{figure}

\begin{figure}[h!]
         \centering
    \subfloat[][$\beta=0.1$]{\includegraphics[width=0.31\textwidth]{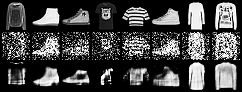}} 
    \hspace{1mm}
    \subfloat[][$\beta=1$]{\includegraphics[width=0.31\textwidth]{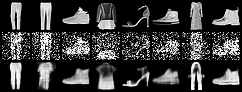}} 
    \hspace{1mm}
    \subfloat[][$\beta=10$]{\includegraphics[width=0.31\textwidth]{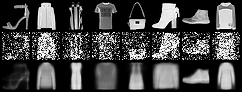}}

    \caption{ We show reconstructions of noisy data for VAEs trained with $\beta\in\{0.1,1,10\}$ on Fashion-MNIST.
    The first row corresponds to the original image, the second to noised a image $\v{x} + \v{\epsilon}$ where $\v{\epsilon} \sim \mathcal{N}(0, (0.5^2)\v{I})$.
    Clearly larger $\beta$ models are less sensitive to noise, supporting our analysis that increasing $\beta$ increases the margin of robustness to perturbations. }

\end{figure}

\clearpage
\newpage
\section{Network Hyperparameters}

All networks used the same hyperparameters.
Namely networks were trained for 100 epochs with the Adam optimizer, with a learning rate of 0.001 and a batch size of 512. 

For MNIST and fashion-MNIST networks for the encoder variance and encoder mean were two hidden layer multi-layer perceptrons (MLPs) with 400 units per layer, which shared their first layer. 
Similarly the decoder was a two layer MLP with 400 units per layer. 
For these datasets we used a latent space size of 20. 

For CIFAR10 we used 4-layer MLPs with 400 units per layer for the  encoder and decoder networks and used a 64-dimensional latent space.

\end{document}